\documentclass[times, utf8, diplomski, english]{fer}
\usepackage{booktabs}
\usepackage{amsmath}
\usepackage{amsthm}
\usepackage{mathtools}
\usepackage{float}
\usepackage{varwidth} % needed for path equiv relations to display properly in
\usepackage{subcaption}
\usepackage{graphicx}
\usepackage[toc,page]{appendix}
\usepackage[percent]{overpic}

\usepackage{hyperref}
\hypersetup{
    colorlinks,
    citecolor=black,
    filecolor=black,
    linkcolor=black,
    urlcolor=black
}

\graphicspath{{img/}}
\usepackage{tikz}

\usetikzlibrary{
	cd,
	shapes.geometric,
	decorations.markings,
	decorations.pathmorphing,
	positioning,
	arrows,
	shapes,
	calc,
	fit,
	quotes}

\tikzstyle{none}=[inner sep=0pt]
\tikzstyle{ibox}=[draw, rounded corners, minimum width = 30pt, minimum height =
18pt, thick]
\tikzstyle{update}=[->,>=stealth, very thick,decoration={snake, pre length =3pt, post
length =3pt},decorate]

\tikzset{
   oriented WD/.style={%everything after equals replaces "oriented WD" in key.
      every to/.style={out=0,in=180,draw},
      label/.style={
         font=\everymath\expandafter{\the\everymath\scriptstyle},
         inner sep=0pt,
         node distance=2pt and -2pt},
      semithick,
      node distance=1 and 1,
      decoration={markings, mark=at position .5 with {\arrow{stealth};}},
      ar/.style={postaction={decorate}},
      execute at begin picture={\tikzset{
         x=\bbx, y=\bby,
         every fit/.style={inner xsep=\bbx, inner ysep=\bby}}}
      },
   bbx/.store in=\bbx,
   bbx = 1.5cm,
   bby/.store in=\bby,
   bby = 1.75ex,
   bb port sep/.store in=\bbportsep,
   bb port sep=2,
   bb port length/.store in=\bbportlen,
   bb port length=0pt,
   bb penetrate/.store in=\bbpenetrate,
   bb penetrate=0pt,
   bb min width/.store in=\bbminwidth,
   bb min width=1cm,
   bb rounded corners/.store in=\bbcorners,
   bb rounded corners=5pt,
   bb small/.style={bb port sep=1, bb port length=2.5pt, bbx=.4cm, bb min width=.4cm, bby=.7ex},
   bb/.code 2 args={%When you see this key, run the code below:
      \pgfmathsetlengthmacro{\bbheight}{\bbportsep * (max(#1,#2)) * \bby}
      \pgfkeysalso{draw,minimum height=\bbheight,minimum width=\bbminwidth,outer sep=0pt,
         rounded corners=\bbcorners,thick,
         prefix after command={\pgfextra{\let\fixname\tikzlastnode}},
         append after command={\pgfextra{\draw
            \ifnum #1=0{} \else foreach \i in {1,...,#1} {
               ($($(\fixname.north
	       west)+(0,.9\bbportsep)$)!{\i/(#1+1)}!($(\fixname.south
	       west)-(0,.9\bbportsep)$)$)
	       +(-\bbportlen,0) coordinate (\fixname_in\i) -- +(\bbpenetrate,0) coordinate (\fixname_in\i')}\fi %Define the endpoints of tickmarks
            \ifnum #2=0{} \else foreach \i in {1,...,#2} {
               ($($(\fixname.north
	       east)+(0,\bbportsep)$)!{\i/(#2+1)}!($(\fixname.south
	       east)-(0,\bbportsep)$)$) +(-\bbpenetrate,0) coordinate (\fixname_out\i') -- +(\bbportlen,0) coordinate (\fixname_out\i)}\fi;
         }}}
   },
   bb name/.style={append after command={\pgfextra{\node[anchor=north] at
(\fixname.north) {#1};}}},
   ibb port sep/.store in=\ibbportsep,
   ibb port sep=2,
   ibb port length/.store in=\ibbportlen,
   ibb port length=4pt,
   ibb min width/.store in=\ibbminwidth,
   ibb min width=1cm,
   ibb rounded corners/.store in=\ibbcorners,
   ibb rounded corners=1pt,
   ibb small/.style={ibb port sep=1, ibb port length=2.5pt, bbx=.4cm, ibb min width=.4cm, bby=.7ex},
   ibb/.code 2 args={%When you see this key, run the code below:
	   \pgfmathsetlengthmacro{\ibbheight}{\ibbportsep * (max(#1,#2)) * \bby}
	   \pgfkeysalso{draw,color=gray!50,minimum height=\ibbheight,minimum width=\ibbminwidth,outer sep=0pt,
		   rounded corners=\ibbcorners,thick,
		   prefix after command={\pgfextra{\let\fixname\tikzlastnode}},
		   append after command={\pgfextra{\coordinate
			   \ifnum #1=0{} \else foreach \i in {1,...,#1} {
				   ($($(\fixname.north
					west)+(0,.9\ibbportsep)$)!{\i/(#1+1)}!($(\fixname.south
						west)-(0,.9\ibbportsep)$)$)
					   +(-\ibbportlen,0) coordinate (\fixname_in\i) -- +(\ibbportlen,0) coordinate (\fixname_in\i')}\fi %Define the endpoints of tickmarks
					   \ifnum #2=0{} \else foreach \i in {1,...,#2} {
						   ($($(\fixname.north
							east)+(0,\ibbportsep)$)!{\i/(#2+1)}!($(\fixname.south
								east)-(0,\ibbportsep)$)$) +(-\ibbportlen,0) coordinate (\fixname_out\i') -- +(\ibbportlen,0) coordinate (\fixname_out\i)}\fi;
		   }}}
   },
   ibb name/.style={append after command={\pgfextra{\node[anchor=north] at
   (\fixname.north) {#1};}}},
   blankbb port sep/.store in=\blankbbportsep,
   blankbb port sep=2,
   blankbb min width/.store in=\blankbbminwidth,
   blankbb min width=1cm,
   blankbb rounded corners/.store in=\blankbbcorners,
   blankbb rounded corners=1pt,
   blankbb small/.style={blankbb port sep=1, blankbb port length=2.5pt, bbx=.4cm, blankbb min width=.4cm, bby=.7ex},
   blankbb/.code 2 args={%When you see this key, run the code below:
	   \pgfmathsetlengthmacro{\blankbbheight}{\blankbbportsep * (max(#1,#2)) * \bby}
	   \pgfkeysalso{draw,color=gray!50,minimum height=\blankbbheight,minimum width=\blankbbminwidth,outer sep=0pt,
		   rounded corners=\blankbbcorners,thick,
		   prefix after command={\pgfextra{\let\fixname\tikzlastnode}},
		   append after command={\pgfextra{\draw
			\ifnum #1=0{} \else foreach \i in {1,...,#1} {
			   ($($(\fixname.north
			   west)+(0,.9\ibbportsep)$)!{\i/(#1+1)}!($(\fixname.south
			   west)-(0,.9\ibbportsep)$)$)
					    coordinate (\fixname_in\i)}\fi %Define tickmarks
			\ifnum #2=0{} \else foreach \i in {1,...,#2} {
			   ($($(\fixname.north
			   east)+(0,.9\ibbportsep)$)!{\i/(#2+1)}!($(\fixname.south
			   east)-(0,.9\ibbportsep)$)$) coordinate (\fixname_out\i)}\fi;
		   }}}
   },
   blankbb name/.style={append after command={\pgfextra{\node[anchor=north] at
     (\fixname.north) {#1};}}},
   symbb port sep/.store in=\symbbportsep,
   symbb port sep=2,
   symbb port length/.store in=\symbbportlen,
   symbb port length=0pt,
   symbb min width/.store in=\symbbminwidth,
   symbb min width=1cm,
   symbb rounded corners/.store in=\symbbcorners,
   symbb rounded corners=5pt,
   symbb small/.style={symbb port sep=1, symbb port length=2.5pt, symbbx=.4cm, symbb min width=.4cm, symbby=.7ex},
   symbb/.code 2 args={%When you see this key, run the code below:
      \pgfmathsetlengthmacro{\symbbheight}{\symbbportsep * (max(#1,#2)) * \bby}
      \pgfkeysalso{draw,minimum height=\symbbheight,minimum width=\symbbminwidth,outer sep=0pt,
         rounded corners=\symbbcorners,thick,
         prefix after command={\pgfextra{\let\fixname\tikzlastnode}},
         append after command={\pgfextra{\draw
            \ifnum #1=0{} \else foreach \i in {1,...,#1} {
               ($($(\fixname.north
	       west)+(0,.9\symbbportsep)$)!{\i/(#1+1)}!($(\fixname.south
	       west)-(0,.9\symbbportsep)$)$)
	       +(-\symbbportlen,0) coordinate (\fixname_in\i) -- +(\symbbportlen,0) coordinate (\fixname_in\i')}\fi %Define the endpoints of tickmarks
            \ifnum #2=0{} \else foreach \i in {1,...,#2} {
               ($($(\fixname.north
	       east)+(0,.9\symbbportsep)$)!{\i/(#2+1)}!($(\fixname.south
	       east)-(0,.9\symbbportsep)$)$) +(-\symbbportlen,0) coordinate (\fixname_out\i') -- +(\symbbportlen,0) coordinate (\fixname_out\i)}\fi;
         }}}
   },
   symbb name/.style={append after command={\pgfextra{\node[anchor=north] at
(\fixname.north) {#1};}}},
}

\newcommand{\LMO}[2][over]{\ifthenelse{\equal{#1}{over}}{\overset{#2}{\bullet}}{\underset{#2}{\bullet}}}
\newcommand{\LTO}[2][\bullet]{\overset{\tn{#2}}{#1}}
\newcommand{\tn}[1]{\textnormal{#1}}
\newcommand{\boxCD}[2][black]{\fcolorbox{#1}{white}{\begin{varwidth}{\textwidth}\centering #2\end{varwidth}}}

\newcommand{\cat}[1]{\mathcal{#1}}%a generic category
\newcommand{\Cat}[1]{\mathbf{#1}}%a named category
\newcommand{\Fun}[1]{#1}%functor

\newcommand{\RR}{\mathbb{R}}
\newcommand{\free}{\Cat{Free}}

\newcommand{\euc}{\mathbf{Euc}}

\newcommand{\gen}[1]{\mathsf{Gen}_{#1}}

\newcommand{\set}{\Cat{Set}}
\newcommand{\zo}{[0, 1]^{100}}
\newcommand{\cc}{\cat{C}}
\newcommand{\dd}{\cat{D}}

\newcommand{\vv}{\cat{V}}
\newcommand{\discr}[1]{\vert{#1}\vert}
\newcommand{\ccv}{\discr{\cc}}
\newcommand{\gimg}{[0, 1]^{64 \times 64 \times 3}}
\newcommand{\sgzo}{[0, 1]^{512}}
\newcommand{\sgimg}{[0, 1]^{1024 \times 1024 \times 3}}
\newcommand{\ih}{I_{H_p}}

\newcommand{\fr}{\free(G)}
\newcommand{\qfr}{{\free(G)/\sim}}
\newcommand{\vfr}{\discr{\fr}}

\newcommand{\para}{\Cat{Para}}
\newcommand{\parae}{\para(\Cat{Euc})}
\newcommand{\parav}{\para(\vv)}

\newcommand{\ps}{\mathcal{P}}

\newcommand{\emb}{E}
\newcommand{\cpt}{\mathfrak{C}}
\newcommand{\task}{{(G, \sim, \vfr \xrightarrow{D_E} \set)}}

\newcommand{\bd}{\mathbf{D}}
\newcommand{\bg}{\mathbf{G}}

\newcommand{\arch}{\mathsf{Arch}}
\newcommand{\pspec}{\mathsf{PSpec}}
\newcommand{\model}{\mathsf{Model}}

\newtheorem{theorem}{Theorem}
\newtheorem{proposition}[theorem]{Proposition}
\newtheorem{example}[theorem]{Example}

\newtheorem{definition}[theorem]{Definition}
\newtheorem{remark}[theorem]{Remark}

\newtheorem{corollary}[theorem]{Corollary}

\makeatletter
\newcommand*{\doublerightarrow}[2]{\mathrel{
  \settowidth{\@tempdima}{$\scriptstyle#1$}
  \settowidth{\@tempdimb}{$\scriptstyle#2$}
  \ifdim\@tempdimb>\@tempdima \@tempdima=\@tempdimb\fi
  \mathop{\vcenter{
    \offinterlineskip\ialign{\hbox to\dimexpr\@tempdima+1em{##}\cr
    \rightarrowfill\cr\noalign{\kern.5ex}
    \rightarrowfill\cr}}}\limits^{\!#1}_{\!#2}}}
\newcommand*{\triplerightarrow}[1]{\mathrel{
  \settowidth{\@tempdima}{$\scriptstyle#1$}
  \mathop{\vcenter{
    \offinterlineskip\ialign{\hbox to\dimexpr\@tempdima+1em{##}\cr
    \rightarrowfill\cr\noalign{\kern.5ex}
    \rightarrowfill\cr\noalign{\kern.5ex}
    \rightarrowfill\cr}}}\limits^{\!#1}}}
\makeatother

\begin{document}

\thesisnumber{2034}

\title{Compositional Deep Learning}

\author{Bruno Gavranović}

\maketitle

% Ispis stranice s napomenom o umetanju izvornika rada. Uklonite naredbu \izvornik ako želite izbaciti tu stranicu.
% \izvornik

\zahvala{I've had an amazing time these last few years. I've had my eyes opened
  to a new and profound way of understanding the world and learned a bunch of
  category theory.
  This thesis was shaped with help of a number of people who I owe my gratitude to.
  I thank my advisor Jan Šnajder for introducing me to machine learning, Haskell and
  being a great advisor throughout these years.
  I thank Martin Tutek and Siniša Šegvić who have been great help for discussing
  matters related to deep learning and for proofchecking early versions of these ideas.
  David Spivak has generously answered many of my questions about categorical concepts related to this thesis.
  I thank Alexander Poddubny for stimulating conversations and valuable
  theoretic insights, and guidance in thinking about these things without whom many of the constructions in this thesis would not be in their current form.
  I also thank Mario Roman, Carles Sáez, Ammar Husain, Tom Gebhart for valuable
  input on a rough draft of this thesis.

  Finally, I owe everything I have done to my brother and my parents for their
  unconditional love and understanding throughout all these years. Thank you.
}

\tableofcontents

\chapter{Introduction}

Our understanding of intelligent systems which augment and automate various
aspects of our cognition has seen rapid progress in recent decades. Partially
prompted by advances in hardware, the field of study of multi-layer artificial neural
networks -- also known as \textit{deep learning} -- has seen astonishing
progress, both in terms of theoretical advances and practical integration with
the real world. Just as mechanical muscles spawned by the Industrial revolution automated significant
portions of human manual labor, so are mechanical minds brought forth by modern deep
learning showing the potential to automate aspects of cognition and pattern
recognition previously thought to have been unique only to humans and animals.

In order to design and scale such sophisticated systems, we need to take extra care when managing their complexity. As with all modern software, their design needs to be done with the principle of \textit{compositionality} in mind.
Although at a first glance it might seem like an interesting yet obscure concept,
the notion of compositionality is at the heart of modern computer science,
especially type theory and functional programming.

Compositionality describes and quantifies how complex things can be
assembled out of simpler parts.
It is a principle which tells us that the design of abstractions
in a system needs to be done in such a way that we can intentionally forget their
internal structure \citep{OnCompositionality}.
This is tightly related to the notion of a leaky abstraction
\citep{LeakyAbstractions} -- a system whose internal design affects its users in
ways not specified by its interface.
% By doing it in a principled way we minimize unwanted effects known as leaky abstractions \citep{LeakyAbstractions}. we can control unwanted emergent effects and ...
% It allows us to control emergent effects and it seems strictly necessary for
% working at scale.

Indeed, going back to deep learning, we observe two interesting
properties of neural networks related to compositionality: (i) they are
compositional – increasing the number of layers tends to yield better
performance, and (ii) they are discovering (compositional) structures in data.
Furthermore, an increasing number of components of a modern deep learning system is learned.
For instance, Generative Adversarial Networks \citep{GAN} learn the \textit{cost
  function}. The paper \textit{Learning to Learn by gradient descent by gradient
descent} \citep{LTL} specifies networks that learn the \textit{optimization function}. The
paper \textit{Decoupled Neural Interfaces using Synthetic Gradients}
\citep{SyntheticGradients} specifies how gradients themselves can be learned.
The neural network system in \citep{SyntheticGradients} can be thought of as a cooperative multi-player game, where some players depend on other ones to learn but can be trained in an asynchronous manner.

These are just rough examples, but they give a sense of things to come.
As more and more components of these systems stop being fixed throughout
training, there is an increasingly larger need for more precise formal
specification of the things that \textit{do} stay fixed.
This is not an easy task; the invariants across all these networks seem to be
rather abstract and hard to describe.

In this thesis we explore the hypothesis that category theory -- a formal
language for describing general abstract structures in mathematics -- could
be well suited to describe these systems in a precise manner.
In what follows we lay out the beginnings of a formal compositional framework
for reasoning about a number of components of modern deep learning architectures. As such the general aim of is thesis is
twofold. Firstly, we translate a collection of abstractions known to machine learning
practitioners into the language of category theory. By doing so we hope to
uncover and disentangle some of the rich conceptual structure underpinning
gradient-based optimization and provide mathematicians with some interesting new problems to solve.
Secondly, we use this abstract category-theoretic framework to conceive a new and practical way to train neural networks and to perform a novel task of
object deletion and insertion in images with unpaired data.

The rest of the thesis is organized as follows.
In Chapter \ref{ch:background} we outline some recent work in neural networks
and provide a sense of depth to which category theory is used in this thesis.
We also motivate our approach by noting a surprising correspondence between a class of neural network architectures and database systems.
Chapter \ref{ch:cat_deep_learning} contains the meat of the thesis and most of
the formal categorical structure. We provide a notion of generalization of
parametrization using the construction we call $\para$. Similarly, provision of categorical
analogues of neural network architectures as functors allows us to generalize
parameter spaces of these network architectures to parameter space of
\textit{functors}.
This chapter concludes with a description of the optimization process in such a
setting. In Chapter \ref{ch:examples} we show how existing
neural network architectures fit into this framework and we conceive a novel network
architecture. In the final chapter we report our
experiments of this novel neural network architecture on some concrete datasets.

% \begin{center}
% \line(1,0){450}
% \end{center}
% Even though the correspondence between database systems and a
% specific class of neural architectures ended up less central to this thesis than
% originally anticipated, all the results in this thesis are the result of almost
% a year-long journey into category theory with this correspondence being the main
% objective. It has been a great experience and I have learned a lot from many people. As
% such, some of the constructions in this thesis have emerged throughout some
% discussions with various people. In particular, the construction $\para$ in
% section \ref{sec:parametricity} arose throughout discussion with David
% Spivak and Julian Hedges, while 
% 
% As most of these constructions were put in their place in time recent to
% the completion of this thesis, there still might be rough edges and obvious
% things that I have missed. Feel free 
% Many of the constructions could be improved and
% There might be obvious things I missed.

\chapter{Background}\label{ch:background}

In this chapter we give a brief overview of the necessary background related to
neural networks and category theory, along with an outline of the used notation and
terminology. In Section \ref{sec:motivation} we motivate our approach by
informally presenting categorical database systems as they pertain to the topic
of this thesis. Lastly, we outline the main results of the thesis.

\section{Neural Networks}

Neural networks have become an increasingly popular tool for solving many
real-world problems. They are a general framework for differentiable optimization which includes many other machine learning approaches as special
cases.

Recent advances in neural networks describe and
quantify the process of discovering high-level, abstract structure in data using
gradient information. As such, learning inter-domain mappings has received
increasing attention in recent years, especially in the context of
\textit{unpaired data} and image-to-image translation \citep{CycleGAN,
  AugmentedCycleGAN}.
\textit{Pairedness} of datasets $X$ and $Y$ generally refers to the existence of some
invertible function $X \rightarrow Y$. Note that in classification we might also
refer to the input dataset as being paired with the dataset of labels, although the
meaning is slightly different as we cannot obviously invert a label $f(x)$ for
some $x \in X$. 

Obtaining this pairing information for datasets usually requires
considerable effort.
Consider the task of object removal from images; obtaining pairs of images where one of
them lacks a certain object, with everything else the same, is much more
difficult than the mere task of obtaining two sets of images: one that contains
that object and one that does not, with everything in these images possibly varying.
Moreover, we further motivate this example by the reminiscence of the way humans
reason about the missing object: simply by observing two unpaired sets of
images, where we are told one set of images lack an object, we are able to learn how the missing object looks like.

There is one notable neural network architecture related to generative modelling
-- Generative Adversarial Networks (GANs) \citep{GAN}.
Generative Adversarial Networks present a radically different approach to training neural networks.
A GAN is a generative model which is a composition of two networks, one called
\textit{the generator} and one called \textit{the discriminator}.
Instead of having the cost function fixed, GAN \textit{learns} the cost function
using the discriminator.
The generator and the discriminator have opposing goals and are trained in an alternating fashion where both
continue improving until the generator learns the underlying data distribution.
GANs show great potential for the development of accurate generative models for complex distributions, such as the distribution of images of written digits or faces. Consequently, in just a few years, GANs have grown into a major topic of research in machine learning.

Motivated by the success of Generative Adversarial Networks in image
generation, existing unsupervised learning methods such as CycleGAN
\citep{CycleGAN} and Augmented CycleGAN \citep{AugmentedCycleGAN} use adversarial
losses to learn the true data distribution of given domains of natural
images and \textit{cycle-consistency} losses to learn \textit{coherent} mappings between those domains.
CycleGAN is an architecture which learns a one-to-one mapping between two domains.
Each domain has an associated \textit{discriminator}, while the mappings between these domains correspond to \textit{generators}. The generators in CycleGAN are a collection of neural networks which is closed under composition, and whose inductive bias is increased by enforcing composition invariants, i.e.~cycle-consistencies. 

Augmented CycleGAN \citep{AugmentedCycleGAN} notices that most relationships
across domains are more complex than simple isomorphisms. It is a
generalization of CycleGAN which learns \textit{many-to-many} mappings between
two domains. Augmented CycleGAN augments each domain with an auxiliary latent
variable and extends the training procedure to these augmented spaces.

\section{Category Theory}

This work builds on the foundations of a branch of
mathematics called Category theory. Describing category theory in just a few
paragraphs is not an easy task as there exist a large number of
equally valid vantage points to observe it from \citep{VantagePoints}.
Rather, we give some intuition and show how it is becoming an unifying force
throughout sciences \citep{RosettaStone}, in all the places in which we need to reason about compositionality.

First and foremost, category theory is a language - a rigorous and a formal one. 
We mean this in the full definition of the word \textit{language} -- it enables
us to specify and communicate complex ideas in a succinct manner. Just as any language -- it \textit{guides and structures thought}.

% Unlike ordinary languages, it is precise and unambiguous, thus reducing
% the need for fuzzy definitions using a language with ad-hoc design rules such as
% English.
% Somebody fluent in this language has a
% powerful tool at this disposal allowing them to efficiently compress a specific
% thought down to its essence.

It is a toolset for describing general abstract structures in mathematics.
Called also ``the architecture of mathematics''
\citep{HigherDimensionalCategoryTheory}, it can be regarded as the \textit{theory of
  theories}, a tool for organizing and layering abstractions and finding formal connections between seemingly disparate fields \citep{SevenSketches}.
Originating in algebraic topology, it has not been designed with the
compositionality in mind.
However, category theory seems to be deeply rooted in all the places we need to
reason about composition.

As such, category theory is slowly finding applications outside of the realm of
pure mathematics. Same categorical structures have been emerging across the sciences: in Bayesian networks \citep{BayesianCTFong, BayesianCTCulbertson}, database systems
\citep{DatabaseOfCategories, RelationalDatabasesIndexedCategories,
  FunctorialDataMigration, AlgebraicDatabases}, version control
systems \citep{PatchTheory}, type theory \citep{CTTypeTheory}, electric circuits
\citep{ElectricCircuitsCT}, natural language processing \citep{NLPCT, NLPCT2},
game theory \citep{CompositionalGameTheory, FunctorialLanguageGames,
  MorphismsOfOpenGames, CoherenceOpenGames}, and automatic differentiation
\citep{SimpleAD}, not to mention its increased use in quantum physics
\citep{PicturingQuantumProcesses, QPhysics1, CategoricalQuantumMechanics}.

In the context of this thesis we focus on categorical formulations of neural networks \citep{BackpropAsFunctor, LearningInvariantsCT, SimpleAD} and databases \citep{FunctorialDataMigration}.
Perhaps the most relevant to this thesis is compositional formulation of
supervised learning found in \citep{BackpropAsFunctor}, whose construction
$\para$ we generalize in Section \ref{sec:parametricity}.

\subsection{Assumptions, Notation, and Terminology}

We assume a working knowledge of fundamental category theory.
Although most of the notation we use is standard, we outline some of the conventions here.

For any category $\cc$ we denote the set of its objects with $Ob(\cc)$ and
individual objects using uppercase letters such as $A, B$, and $C$.
The hom-set of morphisms between two objects $A$ and $B$ in a category $\cc$ is written
as $\cc(A, B)$.
When we want to consider a discretization of a category $\cc$ such that the only
morphisms are the identity morphisms we will write $\ccv$.
We use $A \subseteq B$ to denote $A$ is a subset of $B$, but also more generally
to denote $A$ is a subobject of $B$.
Given some function $f : P \times A \rightarrow B$ and a $p \in P$, we write a
partial application of the $p$ to the first argument of $f$ as $f(p, -) : A
\rightarrow B$.

Given categories $\cc$ and $\dd$, we write $\dd^{\cc}$ for the functor category whose
objects are functors $\cc \rightarrow \dd$ and morphisms are natural
transformations between such functors.

When talking about a monoidal category $\cc$ we will use $\otimes: \cc \times
\cc \rightarrow \cc$ for the monoidal product, $I \in Ob(\cc)$ for the unit
object, $\alpha_{A, B, C}: (A \otimes B) \otimes C \rightarrow A \otimes (B
\otimes C)$ for the associator, and $\lambda_A: I \otimes A \rightarrow A$ for the left unitor.

A notable category we use is $\euc$, the strict symmetric monoidal category
whose objects are finite-dimensional Euclidean spaces and morphisms are
differentiable maps. A monoidal product on $\euc$ is given by the Cartesian product.

Given a directed multigraph $G$, we will write $\fr$ for the free category on that
graph $G$ and $\qfr$ for its quotient category by some congruence relation $\sim$.
Lastly, given a category $\cc$ with generators $G$, we write
$\gen{\cc}$ for the set of generating morphisms $G$ in $\cc$.

\section{Database Schemas as Neural Network Schemas}\label{sec:motivation}

In this section we motivate our approach by highlighting a remarkable correspondence
between database systems, as defined by \citet{FunctorialDataMigration}, and a
class of neural network architectures, here exemplified by CycleGAN
\citep{CycleGAN}, but developed in generality in Chapter \ref{ch:cat_deep_learning}.

Our aim in this section is to present an informal, high-level overview of this correspondence. We hope the emphasized intuition in this section serves as a guide for Chapter \ref{ch:cat_deep_learning} where these structures are described in a formal manner.

The categorical formulation of databases found in
\citep{FunctorialDataMigration} can roughly be described as follows.
A database is modelled as a category which holds just the abstract relationship
between concepts, and a structure-preserving map into another category which holds
the actual data corresponding to those concepts. That is, a database
\textit{schema} $\cc$ specifies a reference structure for a given database instance.
An example, adapted from \cite{SevenSketches}, is shown in Figure
\ref{fig:beatles_schema}.

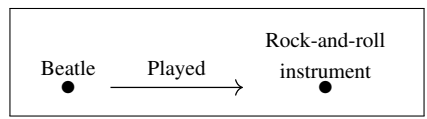
\begin{figure}[H]
  \centering
  \[
    \fbox{
      \begin{tikzcd}[ampersand replacement=\&, column sep=50pt]
        \LTO{Beatle}\ar[r, "\text{Played}"]\&\LTO{\parbox{.7in}{\centering
            Rock-and-roll\\\vspace{-.1in}instrument}}
      \end{tikzcd}
    }
  \]
  
  \caption{Toy example of a database schema}
  \label{fig:beatles_schema}
\end{figure}

Actual data corresponding to such a schema is a functor $\cc \rightarrow \set$, shown in Figure \ref{tab:beatles_data} as a system of interlocking tables. 

\begin{table}[H]
  \centering
  \[
    \begin{tabular}{ c | c}
      \textbf{Beatle}&\textbf{Played}\\\hline
      George&Lead guitar\\
      John&Rhythm guitar\\
      Paul&Bass guitar\\
      Ringo&Drums\\
      ~
    \end{tabular}
    \hspace{1in}
    \begin{tabular}{ c |}
      \textbf{Rock-and-roll instrument}\\\hline
      Bass guitar\\
      Drums\\
      Keyboard\\
      Lead guitar\\
      Rhythm guitar
    \end{tabular}
  \]
  
  \caption{Toy example of a database instance corresponding to the schema in
    Figure \ref{fig:beatles_schema}}
  \label{tab:beatles_data}
\end{table}

Observe the following: The actual data -- sets and functions between them -- are available or known beforehand. There might be some missing data, but all functions
usually have well-defined implementations.

We contrast this with \textit{machine learning}, where we might have plenty of
data, but no known implementation of functions that map between data samples.
Table \ref{tab:paired_data} shows an example from the setting of supervised
learning. In this example, samples are paired: for every input we have an
expected output.
Thus, given a trained model and a new sample from a test set -- say ``DataSample4717'' -- we hope our model has learned to generalize well and assign a corresponding output to this input.

\begin{table}[H]
  \centering
  \[
    \begin{tabular}{ c | c}
      \textbf{Input}&\textbf{Corresponding output}\\\hline
      DataSample1&ExpectedOutput1 \\
      DataSample17&ExpectedOutput17 \\
      DataSample30&ExpectedOutput30 \\
      DataSample400&ExpectedOutput400 \\
      $\dots$ \\
    \end{tabular}
    \hspace{.7in}
    \begin{tabular}{ c |}
      \textbf{Output}\\\hline
      ExpectedOutput1 \\
      ExpectedOutput17 \\
      ExpectedOutput30 \\
      ExpectedOutput400 \\
      $\dots$ \\
    \end{tabular}
  \]
  
  \caption{Example of paired datasets in a setting of supervised learning}
  \label{tab:paired_data}
\end{table}

Moreover, we point out that the luxury of having paired data is not always at
our disposal: real life data is mostly \textit{unpaired}. We might have two
datasets that are related \textit{somehow}, but without knowing if any inputs
match any of the outputs.
Consider the example shown in Table \ref{tab:unpaired_data}.

\begin{table}[H]
  \centering
  \[
    \begin{tabular}{ c | c }
      \textbf{Horse image}&\textbf{Horse $\rightarrow$ Zebra}\\\hline
      HorseImg1& ? \\
      HorseImg24& ? \\
      $\vdots$  \\
      HorseImg303& ? \\
      HorseImg2392& ? \\
    \end{tabular}
    \hspace{.7in}
    \begin{tabular}{ c | c}
      \textbf{Zebra image}&\textbf{Zebra $\rightarrow$ Horse}\\\hline
      ZebraImg10& ? \\
      ZebraImg430& ? \\
      $\vdots$ \\
      ZebraImg566& ? \\
      ZebraImg637& ? \\
      ZebraImg700& ? \\
      $\vdots$
    \end{tabular}
  \]
  
  \caption{An example of two unpaired datasets}
  \label{tab:unpaired_data}
\end{table}

The example in Table \ref{tab:unpaired_data} is given by the schema in Figure
\ref{fig:CycleGAN_schema}. They depict the scenario where we have two image
datasets: of images of horses and of images of zebras. Consider those images as
photographs of these animals in nature, in various positions and from various
angles (Figure \ref{fig:cycleGANfrontpage}).
We just have \textit{some} images and do not necessarily
have any pairs of images in our dataset.

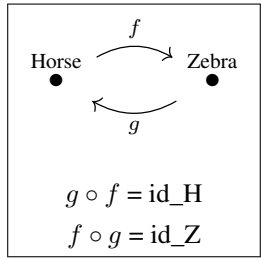
\begin{figure}[H]
    \centering
\boxCD{
\begin{tikzcd}[column sep=small, ampersand replacement=\&]
  \LTO{Horse}\ar[rr, bend left, "f"]\&\&
    \LTO{Zebra}\ar[ll, bend left, "g"]\\
\end{tikzcd}
\\~\\\footnotesize
 $g \circ f$  = \textrm{id}\_{H}  \\
 $f \circ g$ = \textrm{id}\_{Z}
}
\caption{We might hypothesize mappings $f$ and $g$ exist -- without knowing anything about them other than their composition invariants.} 
\label{fig:CycleGAN_schema}
  \end{figure}

Although our dataset does not contain any explicit $\mathrm{Horse-Zebra}$ pairs -- we
still might hypothesize that these pairs exist. In other words, we could think that it should be possible to map back-and-forth between images of Horses and Zebras, just by changing the texture of animal in such an image.
In other words, we posit there exists a specific relationship between
the datasets.
Compared to \textit{databases}, where we have the data and well-defined function
implementations, here all we have is data and a \textit{rough idea} of
which mappings exist, without known implementations.
The issue is that our dataset does not contain explicit pairs usable in the
context of supervised learning.

What is described here is first introduced in a paper by \cite{CycleGAN}. They
introduce a model CycleGAN which is a generalization of Generative Adversarial
Networks \citep{GAN}.
Figure \ref{fig:cycleGANfrontpage} is adapted from their paper and showcases
the main results. 

\begin{figure}[H]
\includegraphics[width=\textwidth]{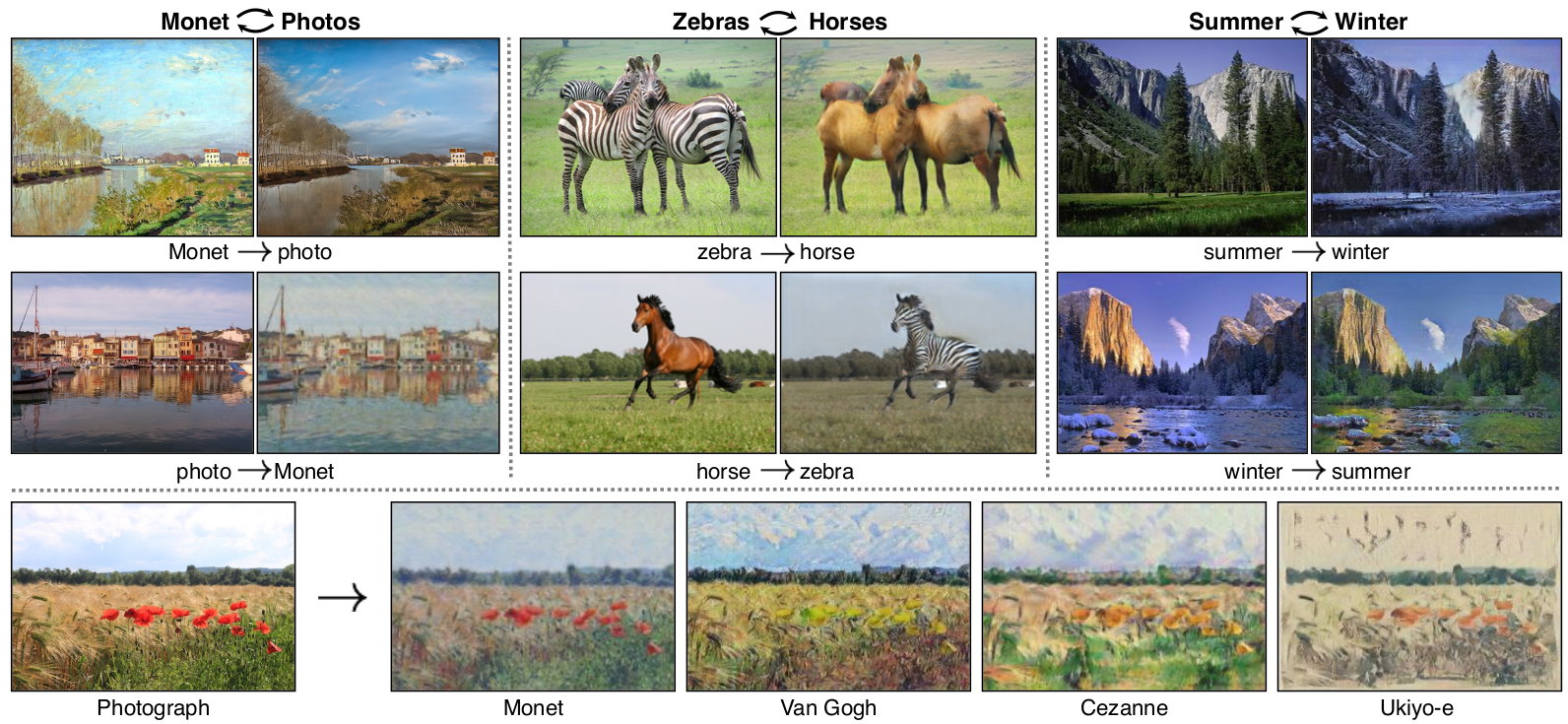}
\centering
\caption{Given any two unordered image collections $X$ and $Y$, CycleGAN learns
  to automatically ``translate'' an image from one into the other and vice
  versa. Figure taken from \cite{CycleGAN}.}
\label{fig:cycleGANfrontpage}
\end{figure}

This shows us that, at least at a first glance, CycleGAN and categorical databases
are related in some abstract way. After developing the necessary categorical tools to reason about CycleGAN, we elaborate on this correspondence in Section \ref{sec:fdm_correspondence}.

\section{Outline of the Main Results}

This thesis aims to bridge two seemingly distinct ideas: category theory and deep
learning.
In doing so we take a collection of abstractions in deep learning and formalize the
notation in categorical terms. 
This allows us to begin to consider a formal theory of gradient-based learning in the
\textit{functor space}.

We package a notion of the interconnection of networks as a free category $\fr$ on
some graph $G$ and specify any equivalences between networks as relations
between morphisms as a quotient category $\qfr$.
Given such a category -- which we call a \textit{schema}, inspired by
\cite{FunctorialDataMigration} -- we specify the architectures of its networks
as a functor $\arch$.
We reason about various other notions found in deep learning, such as datasets,
embeddings, and parameter spaces.
The training process is associated with an indexed family of functors
$\{H_{p_i}: \fr \rightarrow \set \}_{i=1}^T$, where $T$ is the number of training steps and $p_i$ is some choice of a parameter for that architecture at the training step $i$.

Analogous to standard neural networks -- we start with a randomly
initialized $H_p$ and iteratively update it using gradient descent.
Our optimization is guided by \textit{two} objectives.
These objectives arise as a natural generalization of those found in \citep{CycleGAN}.
One of them is the adversarial objective -- the minmax objective found in any
Generative Adversarial Network. The other one is a generalization of the
cycle-consistency loss which we call \textit{path-equivalence loss}.

Although mathematically abstract, this approach yields useful insights.
Our formulation provides maximum generality: (i) it enables learning with
unpaired data as it does not impose any constraints on ordering or pairing of
the sets in a category, and (ii) although specialized to generative models in
the domain of computer vision, the approach is domain-independent and general
enough to hold in any domain of interest, such as sound, text, or video. 

We show that for specific choices of $\qfr$ and the dataset we recover GAN \citep{GAN} and CycleGAN \citep{CycleGAN}.
Furthermore, a novel neural network architecture capable of learning to remove
and insert objects into an image with unpaired data is proposed. We outline its categorical perspective and show it in action by testing it on three different datasets.

% \begin{itemize}
% \item Our main contribution is discovery of a method to regularize neural
%   network training using category theory.
% \item Using categorical coherence conditions to regularize neural network
%   training
% \item This categorical formulation does not tell us anything new, but rather it
%   disentagles a multitude of complex abstractions. We still haven't reaped the
%   payoff of this, which should be apparent in the coming years as ACT takes off
% \item be consistent with bolding in definitions
% \end{itemize}

\chapter{Categorical Deep Learning}\label{ch:cat_deep_learning}

Modern deep learning optimization algorithms can be framed as a gradient-based search in some function space $Y^X$, where $X$ and $Y$ are sets that have been endowed with extra structure.
Given some sets of data points $D_X \subseteq X$, $D_Y \subseteq Y$, a typical
approach for adding inductive bias relies on exploiting this extra structure
associated to the data points embedded in those sets, or those sets themselves. This structure includes domain-specific features which can be exploited by various methods -- convolutions for images, Fourier transform for audio, and
specialized word embeddings for textual data.

In this chapter we develop the categorical tools to increase inductive bias of a
model without relying on any extra such structure.
We build on top of the work of \citep{BackpropAsFunctor, FunctorialDataMigration}
and, very roughly, define our model as a \textit{collection of networks} and increase its
inductive bias by \textit{enforcing their composition invariants}.

% Our approach is based on the idea that we can make a \textit{formal separation}
% of these constructs. They include the notion of interconnectivity pattern of
% networks, the corresponding functions, domains and codomains of those
% functions, as well as any extra structure adjoined to related sets, such as
% that of an Euclidean space. 

\section{Model Schema}\label{sec:model_schema}

Many deep learning models are complex systems, some comprised of several
neural networks. Each neural network can be identified with domain $X$, codomain $Y$, and
a \textit{differentiable parametrized function} $X \rightarrow Y$.
Given a \textit{collection} of such networks, we use a directed multigraph to
capture their interconnections. We use vertices to represent the domains and codomains,
and edges to represent differentiable parametrized functions. Observe that
an ordinary graph will not suffice, as there can be two \textit{different}
differentiable parametrized functions with the same domain and codomain.

Each directed multigraph $G$ gives rise to a corresponding free category on
that graph $\fr$. Based on this construction, Figure \ref{fig:birdseye} shows
the interconnection pattern for generators of two popular neural network architectures: GAN \citep{GAN} and CycleGAN \citep{CycleGAN}.

\begin{figure}[H]
\centering
\begin{subfigure}{.5\textwidth}
\vskip 0.2in
\centering
\boxCD{
\begin{tikzcd}[column sep=small, ampersand replacement=\&]
  \LTO{Latent space}\ar[rr, bend right, "h"]\&\&
  \LTO{Image} \\
\end{tikzcd}
  \\~\\\footnotesize
  \textit{no equations}
}
\caption{GAN}
\label{fig:GAN_schema}
\end{subfigure}%
\begin{subfigure}{0.5\textwidth}
    \centering
\boxCD{
\begin{tikzcd}[column sep=small, ampersand replacement=\&]
  \LTO{Horse}\ar[rr, bend left, "f"]\&\&
    \LTO{Zebra}\ar[ll, bend left, "g"]\\
\end{tikzcd}
\\~\\\footnotesize
 $g \circ f$  = \textrm{id}\_{H}  \\
 $f \circ g$ = \textrm{id}\_{Z}
}
\caption{CycleGAN} 
\label{fig:CycleGAN_schema}
\end{subfigure}
\caption{Bird's-eye view of two popular neural network models}
\label{fig:birdseye}
\end{figure}
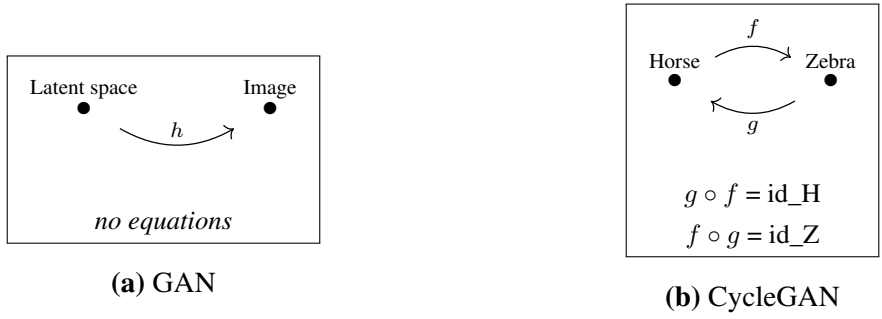

Observe that CycleGAN has some additional properties imposed on it, specified by
equations in Figure \ref{fig:birdseye} (b). These are called CycleGAN cycle-consistency
conditions and can roughly be stated as follows: given domains $A$ and $B$
considered as sets, $a \approx g(f(a)), \forall a \in A$ and $b \approx f(g(b)), \forall b \in B$.

Our approach involves a realization that cycle-consistency conditions
\textit{can be generalized} to path equivalence relations, or, in formal terms -- a congruence relation.
The condition $a \approx g(f(a)), \forall a \in A$ can be reformulated such that
it does not refer to the elements of the set $a \in A$. By \textit{eta-reducing}
the equation we obtain $id_{a} = g \circ f$. Similar reformulation can be done for the other condition: $id_{b} = f \circ g$. 

This allows us to package the newly formed equations as equivalence relations on
the sets $\fr(A, A)$ and $\fr(B, B)$, respectively. This notion can be further packaged
into a quotient category $\qfr$, together with the quotient functor ${\fr
\xrightarrow{Q} \qfr}$.

This formulation -- as a free category on a graph $G$ -- represents the
cornerstone of our approach. These schemas allow us to reason solely about the
interconnections between various concepts, rather than jointly with functions,
networks or other some other sets. All the other constructs in this
thesis are structure-preserving maps between categories whose domain, roughly,
can be traced back to $\fr$.

\section{What Is a Neural Network?}\label{sec:parametricity}

In computer science, the idea of a \textit{neural network} colloquially means a
number of different things. At a most fundamental level, it can be interpreted as
a system of interconnected units called neurons, each of which has a firing
threshold acting as an information filtering system. Drawing inspiration
from biology, this perspective is thoroughly explored in literature. In many
contexts we want to focus on the mathematical properties of a neural network and as such identify it with a function between sets $A \xrightarrow{f}
B$. Those sets are often considered to have extra structure, such as those of Euclidean
spaces or manifolds. Functions are then considered to be maps of a given
differentiability class which preserve such structure.
We also frequently reason about a neural network jointly with its parameter space $P$
as a function of type $f: P \times A \rightarrow B$. For instance, consider a classifier
in the context of supervised learning. A convolutional neural network whose
input is a $32 \times 32$ \texttt{RGB} image and output is real number can be
represented as a function with the following type: $\RR^n \times \RR^{32 \times
  32 \times 3} \rightarrow \RR$, for some $n \in \mathbb{N}$. In this case $\RR^n$ represents the parameter space of this network.

The former ($A \rightarrow B$) and the latter ($P \times A \rightarrow B$)
perspective on neural networks are related.
Namely, consider some function space $B^A$. Any notion of smoothness in such a space is not well defined without any further assumptions on sets $A$ or $B$.
This is the reason deep learning employs a gradient-based search in such a space
via a proxy function $P \times A \rightarrow B$.
This function specifies an entire
\textit{parametrized family} of functions of type $A \rightarrow B$, because partial application of each $p \in P$ yields a function
$f(p, -) : A \rightarrow B$.
This choice of a parametrized family of functions is part of the
\textit{inductive bias} we are building into the training process.
For example, in computer vision it is common to restrict the class of functions
to those that can be modeled by convolutional neural networks.

In more general terms, by currying $f : P \times A \rightarrow B$ we obtain an
important construction in the literature as a parameter-function map
$\mathcal{M} :P \rightarrow B^A$ \citep{ParamFunctionMap}.

The parameter-function map is important as it allows us to map behaviors in the parameter
space to the behavior of functions in the function space.
For example, partial differentiation of $f$ with respect to $p$ allows us
to use gradient information to search the parameter space -- but also the space
of a particular family of functions of type $A \rightarrow B$ specified by $f$.

With this in mind, we recall the model schema.
For each morphism $A \rightarrow B$ in $\fr$ we are interested
in specifying a parametrized function $f : P \times A \rightarrow B$, i.e.~a parametrized
\textit{family of functions} in $\set$.
The construction $f$ describes a neural network architecture, and a choice of a partially applied $p \in P$ to $f$ describes a choice of some parameter value for that specific architecture.

We capture the notion of parametrization with a construction $\para$.
We package both of these notions -- choosing an architecture and choosing
parameters -- into functors. Figure \ref{fig:arch_param}
shows a high-level overview of these constructions -- including a notion that
will be central to this thesis -- $\parae$.

\tikzset{
    labl/.style={anchor=south, rotate=-28, inner sep=.5mm}
}

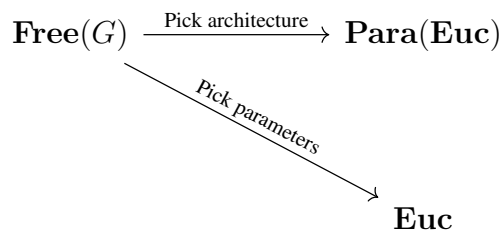
\begin{figure}[H]
\centering
\begin{tikzcd}[column sep = 70pt, row sep = 50pt]
\free(G) \arrow[r, "\textrm{Pick architecture}"] \arrow[rd, "\textrm{Pick
  parameters}" labl] & \parae  \\
& \euc
\end{tikzcd}
\caption{High-level structure of architecture and parameter selection}
\label{fig:arch_param}
\end{figure}

% the following desiderata in mind:
% \begin{itemize}
% \item There exists $p \in P$ such that $f(p, -)$ is the goal of the training 
% \item The training process can find such a $p$.
% \end{itemize}

\section{Parametrization and Architecture}

In this section we begin to provide a rigorous categorical framework for reasoning about
neural networks.

\subsection{Parametrization}

We now turn our attention to a construction $\para$, which allows us to
compose parametrized functions of type $f : P \times A \rightarrow B$ in such a
way that it abstracts away the notion of a parameter.
This is a generalization of the construction $\para$ found in \citep{BackpropAsFunctor}.

\begin{definition}[Para]\label{def:para}
  Given any small symmetric monoidal category $(\vv, I, \otimes)$, we can
  construct another small symmetric monoidal category $(\para(\vv), I, \otimes)$ given by the following:
  \begin{itemize}
  \item \textbf{Objects.} $\para(\vv)$ has the same objects as $\vv$;
  \item \textbf{Morphisms.} $\para(\vv)(A, B) := {\{ f : P \otimes A \rightarrow B
    \mid P \in Ob(\vv)\}/\sim}$, where ${f \sim f'}$ if there exists an isomorphism
    $g \in \vv(P, P')$ such that $f' \circ (g \otimes id_A) = f$;
  \item \textbf{Identity.} Identity of any object $A \in Ob(\para(\vv))$ is the
    left unitor $\lambda_A: I \otimes A \rightarrow A$;
  \item \textbf{Composition.} For every three objects $A, B, C$ and morphisms $f
    : P \otimes A \rightarrow B \in \para(\vv)(A, B)$ and $g : Q \otimes B
    \rightarrow C \in \para(\vv)(B, C)$ for some $P, Q \in Ob(\vv)$, we specify a
    morphism $g \circ f \in \para(\vv)(A, C)$ as follows:
\begin{align*}
g \circ f &:  (P \otimes Q) \otimes A \rightarrow C \\
g \circ f &= \lambda ((p, q), a) \rightarrow g(q, f(p, a)) 
\end{align*}

\end{itemize}

Monoidal structure is inherited from $\vv$.

\end{definition}

\begin{proof}
To prove $\parav$ is indeed a category, we need to show associativity and unitality of
composition strictly holds.
Observe that composition in $\parav$ is defined in terms of the monoidal product
in $\vv$. Consider the two different ways of composing following morphisms in $\parav$:
\begin{align*}
  f&: P \otimes A \rightarrow B \in \parav(A, B),\\
  g&: Q \otimes B \rightarrow C \in \parav(B, C),\\
  h&: R \otimes C \rightarrow D \in \parav(C, D).
\end{align*}
Depending on the order we compose them, we end up with one of the two following morphisms:
\begin{align*}
  h \circ (g \circ f) &: ((P \otimes Q) \otimes R) \otimes A \rightarrow D,\\
  (h \circ g) \circ f &: (P \otimes (Q \otimes R)) \otimes A \rightarrow D.
\end{align*}

\sloppy
Even though it might seem strictness of the associativity of composition in
$\parav$ depends on the strictness of the monoidal product in $\vv$, we note
that morphisms in $\parav$ are actually equivalence classes.
Namely, because every monoidal category comes equipped with an associator $\alpha_{P,
  Q, R}: {(P \otimes Q) \otimes R \cong P \otimes (Q \otimes R)}$, both $h \circ
(g \circ f)$ and $(h \circ g) \circ f$ fall into the same equivalence class,
making composition in $\parav$ strictly associative.

Similar argument can be made for the unitality condition, thus showing $\parav$
is a category.
\end{proof}

We will call morphisms in $\para(\vv)$ \textit{parametrized morphisms},
\textit{neural networks}, or \textit{neural network architectures} depending on
the context.\footnote{
 Note that $\para$ is not a functor between categories $\vv$ and $\parav$, but
 rather an endofunctor on $\Cat{SMC_{str}}$, the category of of all small
 symmetric monoidal categories. We outline this for completness but do not prove
 or explore this direction further.}
Abusing our notation slightly, we will refer to a morphism in $\parav$ by some
elements from the corresponding equivalence class.

The composition of morphisms in $\parav$ is defined in such a
way that it explicitly keeps track of parameters. Namely, when we sequentially compose two morphisms $A \xrightarrow{f} B$ and $B \xrightarrow{g} C$ in $\parav$, we are
actually composing morphisms $P \otimes A \rightarrow B$ and
$Q \otimes B \rightarrow C$ in $\vv$ such that the composition $(P \otimes Q) \otimes A
\rightarrow C$ keeps track of parameters separately.\footnote{
 Even though objects in $\parav$ generally do not have elements, in Definition
 \ref{def:para} composition is stated in terms of elements to supply intuition. We have also bracketed the monoidal product even thought $\otimes$ is strict for a similar reason -- to show we think of $P \otimes Q$ as the parameter of the composite
 $g \circ f$. We invite the reader to \citep{BackpropAsFunctor}, which contains a
 particularly clean interpretation of $\parav$ in terms of string
 diagrams.}

% Monoidal product of $f: P
% \otimes A \rightarrow C \in \parav(A, C)$ and $g: Q \otimes B \rightarrow D \in
% \parav(B, D)$ is a morphism of type $f \otimes g : (P \otimes Q) \otimes (A \otimes B) \rightarrow C \otimes D$. 

% 
% \begin{theorem}
%  $\para(\euc)$ is sufficient to model feedforward and convolutional neural networks. 
% \end{theorem}
% 
% \begin{proof}
%   For feedforward we refer to \cite{BackpropAsFunctor}. Given that convolutional
%   networks can be seen as special cases of feedforward (with some weights
%   coupled and others set to zero), we show it holds for them as well.
% \end{proof}
% This, of course, ignores the spatial properties of convolution and merely
% shows they are modelable by current methods.
%

This construction $\para(\vv)$ generalizes the category $\para$ as originally defined in
\citet{BackpropAsFunctor}. Namely, by setting $\vv := \euc$, we recover the
notion $\para$ as it is described in the aforementioned paper.
As $\para(\euc)$ will make continued appearance in this thesis, we describe some of its properties here.

$\parae$ is a strict symmetric monoidal category whose objects are Euclidean
spaces and morphisms $\RR^n \rightarrow \RR^m$ are equivalence classes of
differentiable function of type $\RR^p \times \RR^n \rightarrow \RR^m$, for some
$p \in \mathbb{N}$.
We refer to $\RR^p$ as the parameter space.

Monoidal product in $\parae$ is the Cartesian product inherited from $\euc$. As
maps in $\euc$ are differentiable, so are maps in $\parae$ thus enabling us to
consider gradient-based optimization in a more abstract setting.

\subsubsection{Parameter selection in a monoidal category}

Previously, in the context of functions $f : P \times A \rightarrow B$ between
sets, we have considered the partial application $f(p, -)$. Now, we are
interested in doing the same in $\parav$, given some small symmetric monoidal
category $\vv$.

There are two issues with this statement.
\begin{enumerate}
\item Internal structure of objects in $\vv$ is unknown -- they might not be sets
  with elements
\item It assumes a specific notion of completeness: for every $f : P \otimes A \rightarrow B$ in $\parav$ we assume $f(p, -) \in \vv(A, B), \forall p \in P$.
\end{enumerate}

We solve both of the issues by noting that picking a parameter $p \in P$ in a monoidal
category without any assumption on the internals of its objects amounts to picking a morphism $I \xrightarrow{p} P$.
Then the specific notion of completeness \emph{is already given to us} by the
monoidal product on $\vv$. Indeed, $f(p, -)$ amounts to the composition $f
\circ (p \otimes id_A) \circ \lambda_A^{-1} : \vv(A, B)$, where $\lambda^{-1}$
is the inverse of left unitor of the monoidal category.

However, most of our consideration in this thesis will be where $\vv := \euc$. In
these cases the notation $p \in P$ is well-justified, as $\euc$ comes equipped
with a forgetful functor $\euc \xrightarrow{U} \set$.

\subsection{Model Architecture}\label{sec:model_arch}

We now formally specify \textit{model architecture} as a functor.
We chose $\fr$ as the domain of the functor, rather than $\qfr$, for reasons
that will be explained in Remark \ref{rem:relations}. As such, observe that the
action on morphisms is defined on the generators in $\fr$.

\begin{definition}\label{def:functor_p}
  Architecture of a model is a functor $\arch: \fr \rightarrow \para(\euc)$.
\begin{itemize}
\item For each $A \in Ob(\fr)$, it specifies an Euclidean space $\RR^a$;
\item For each generating morphism $A \xrightarrow{f} B$ in $\fr$, it specifies
  a morphism \\ $\RR^a \xrightarrow{\arch(f)} \RR^b$ which is a differentiable parametrized function of type $\RR^n \times \RR^a \rightarrow \RR^b$.
\end{itemize}
Given a non-trivial composite morphism $f = f_n \circ f_{n - 1} \circ \dots \circ f_1$ in $\cc$,
the image of $f$ under $\arch$ is the composite of the image of each constituent: 
$\arch(f) = \arch(f_n) \circ \arch(f_{n - 1}) \circ \dots \circ \arch(f_1)$.
$\arch$ maps identities to the projection $\pi_2: I \times A
\rightarrow A$.
\end{definition}

\begin{remark}\label{rem:relations}
  The reason the domain of $\arch$ is $\fr$, rather than $\qfr$ can be illustrated with
  the following example. Consider two morphisms $id_A: A \rightarrow A$ and $g
  \circ f: A \rightarrow A$ in some $\qfr$. Suppose $id_A = g \circ f$. The
  value image of architecture at $id_A$ is already given: $\arch(id_A) := I \otimes A
  \rightarrow A$, but for $g \circ f$ it is defined as $\arch(g) \circ
  \arch(f)$. Hence, there exists a choice $\arch(g)$ and $\arch(f)$ such
  that $\arch(id_A) \neq \arch(g) \circ \arch(f)$, rendering this structure
  defined on $\qfr$ not a functor.
  However, this is not an issue; we will show it will be possible to \textit{learn} those relations.
\end{remark}

The choice of architecture $\fr \xrightarrow{\arch} \parae$ goes hand in hand
with the choice of an \textit{embedding}.

\begin{proposition}
An embedding is a functor $\discr{\fr} \xrightarrow{E} \set$ which agrees with
$\arch$ on objects.
\end{proposition}

Observe that the codomain of $E$ is $\set$, rather than $\parae$. This can be
shown in two steps: (i) $\parae$ and $\euc$ have the same objects, and (ii)
objects in $\euc$ are just sets with extra structure.

Embedding $E$ and $\arch$ come up in two different scenarios. Sometimes we start
out with a choice of architecture which then induces the embedding. In other cases, the
embedding is given to us beforehand and it restricts the possible choice of
architectures.
The embedding construction will prove to be important later in this thesis.

Having defined $\fr \xrightarrow{\arch} \parae$, we shift our attention to the
notion of parameter specification. For a given a differentiable parametrized function
$\arch(f): \RR^n \times \RR^a \rightarrow \RR^b$ the training process involves
repeatedly updating the chosen $p \in \RR^n$. We might suspect this process of
choosing parameters might be made into a functor $\parae \rightarrow \euc$, mapping each $\arch(f): \RR^n \times \RR^a \rightarrow \RR^b$ into $\arch(f)(p, -) : \RR^a \rightarrow \RR^b$. 
However, it can be seen that this is not the case by considering
fully-connected neural networks; we want to specify parameters for an $N$-layer
neural network by specifying parameters for each of its layers. In categorical
terms, this means that we want to specify the action of this functor on generators in $\parae$.
Since $\parae$ might also have arbitrary relations between morphisms, we cannot
be sure this recursive approach will satisfy any such relations. This, in turn, stops us from considering this construction as a functor.

Moreover, even if this construction could be a functor, we show that it might not
be the construction we are interested in.
Observe that $\arch$ is not necessarily faithful. Suppose two different arrows $A
\doublerightarrow{f}{g} B$ in $\fr$ are
mapped to the same neural network architecture $\arch(f) = \arch(g): \RR^n \times \RR^a
\rightarrow \RR^b : \parae(\arch(A), \arch(B))$. Even though images of these
arrows are the same, it is beneficial and necessary to keep in mind that those
two are separate neural networks, each of which could have a different parameter
assigned to it during training. Any such parameter specification functor whose
domain is $\parae$ would have to specify \textit{one} parameter value for such
a morphism. This, in turn, means that this construction would not allow us to have two distinct parameters for what we consider to be two distinct networks.

Before coming back to this issue, we take a slight detour and consider various notions of \textit{parameter spaces}.

\section{Architecture Parameter Space}\label{sec:param_spec}

Each network architecture $f : \RR^n \times \RR^a \rightarrow \RR^b$ comes equipped
with its parameter space $\RR^n$.
Just as $\fr \xrightarrow{\arch} \parae$ is a categorical generalization of
architecture, we now show there exists a categorical generalization of a
parameter space. In this case -- it is the parameter space of a functor.
Before we move on to the main definition, we package the notion of parameter
space of a function $f : \RR^n \times \RR^a \rightarrow \RR^b$ into a simple function
$\mathfrak{p}(f) = \RR^n$.

\sloppy
\begin{definition}[Functor parameter space]\label{def:parameter_space}
Let $\gen{\fr}$ the set of generators in $\fr$. The total parameter map ${\mathcal{P}: Ob(\parae^{\fr}) \rightarrow Ob(\euc)}$ assigns to each functor $\arch$
the product of the parameter spaces of all its generating morphisms:
\[
\mathcal{P}(\arch) = \displaystyle \prod_{f \in \gen{\fr}}{\mathfrak{p}(\arch(f))}
\]
\end{definition}
Essentially, just as $\mathfrak{p}$ returns the parameter space of a function,
$\mathcal{P}$ does the same for a \textit{functor}.

% Given a model architecture $\fr \xrightarrow{\arch} \parae$, for each $p \in
% \mathcal{P}(\arch)$ there exists a \textit{model} $\fr \xrightarrow{\model_i}
% \euc$ model whose parameters can be identified with $p$.
% This is a natural generalization of Proposition \ref{prop:partial_application}.

We are now in a position to talk about parameter specification. Recall the
non-categorical setting: given some network architecture $f: P \times A
\rightarrow B$ and a choice of $p \in \mathfrak{p}(f)$ we can partially apply
the parameter $p$ to the network to get $f(p, -) : A \rightarrow B$.
This admits a straightforward generalization to the categorical setting.

\begin{definition}[Parameter specification]\label{def:param_spec}
Parameter specification $\pspec$ is a dependently typed function with the
following signature:
\[
\pspec : (\arch : Ob(\parae^{\fr})) \times \mathcal{P}(\arch) \rightarrow Ob(\euc^{\fr})
\]
Given an architecture $\arch$ and a parameter choice $(p_f)_{f \in \gen{\fr}} \in \mathcal{P}(\arch)$ for that architecture, it defines a choice of
a functor in $\euc^{\fr}$. This functor acts on objects the same as $\arch$. On
morphisms, it partially applies every $p_f$ to the corresponding morphism
$\arch(f) : \RR^n \times \RR^a \rightarrow \RR^b$, thus yielding $f(p_f, -) : \RR^a \rightarrow \RR^b$ in $\euc$.
\end{definition}

Elements of $\euc^{\fr}$ will play a central role later on in the thesis.
These elements are functors which we will call \textit{Models}. Given some
architecture $\arch$ and a parameter $p \in \ps(\arch)$, a model $\fr \xrightarrow{\model_p} \euc$
generalizes the standard notion of a model in machine learning -- it can be used
for inference and evaluated.

Analogous to database instances in \cite{FunctorialDataMigration}, we call a
\textit{network instance} $H_p$ the composition of some $\model_p$ with the
forgetful functor $\euc \xrightarrow{U} \set$. That is to say, a network instance is a functor $\fr \xrightarrow{H_p} \set := U \circ \model_p$.

\begin{corollary}\label{cor:action_on_objects}
Given an architecture $\fr \xrightarrow{\arch} \parae$, for each $p \in
\ps(\arch)$ all network instances $\fr \xrightarrow{H_p} \set$ act the same on objects.
\end{corollary}

\begin{proof}
As $H_p : U \circ Model_P$, it is sufficient to prove that for all $p \in
\ps(\arch)$ all models act the same on objects. This follows from Definition
\ref{def:param_spec} which tells us that action of $\model_p$ on objects is
independent of $p$.
\end{proof}

This means that the only thing different between any two network instances
during training is the choice of a parameter they partially apply to morphisms in the
image of $\parae$. Objects -- the domain of functions resulting from partial
applications -- stay the same.
This is evident in standard machine learning models as well, where we obviously do not
change their type of input during architecture parameter updates.

We shed some more light on these constructions using Figure
\ref{fig:arch_model_h}.

\begin{figure}[H]
\centering
\begin{tikzcd}[column sep=40pt, row sep=50pt]
\free(G) \arrow[r, "\arch"] \arrow[rd, "\model_p"] \arrow[ddr,
"\Fun{H}_p"] & \parae \\
& \euc \arrow[d, "U"] \\
& \Cat{Set} 
\end{tikzcd}
\caption{$\fr$ is the domain of three types of functors of interest: $\arch$, $\model_p$
  and $H_p$.}
\label{fig:arch_model_h}
\end{figure}
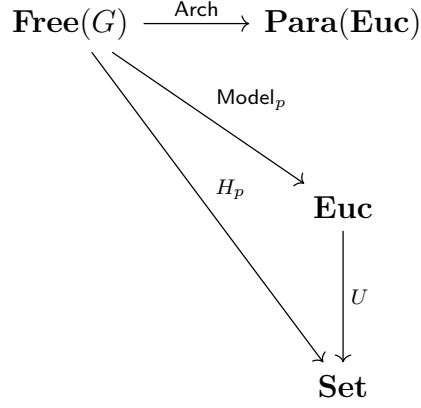

\subsection{Parameter-functor Map.}

In this section we show the parameter-function map detailed in Section
\ref{sec:parametricity} admits a natural generalization to a categorical
setting. Recall that given some architecture -- a differentiable function between sets
$f: P \times A \rightarrow B$ -- we can obtain the \textit{parameter-function map} $P \rightarrow B^A$.

In a categorical setting, given an architecture $\fr \xrightarrow{\arch} \parae$ we can
obtain the \textit{parameter-functor map} by partially applying $\arch$ to $\pspec$.

\begin{definition}[Parameter-functor map]\label{def:param_functor_map}
  Let $\fr \xrightarrow{\arch} \parae$ be the architecture.
  Then the parameter-functor map $\mathcal{M}_{\euc}$ is the partial application
  $\mathcal{M}_{\euc} := \pspec(\arch, -)$ which has the type:
  
  \[
    \mathcal{M}_{\euc} : \mathcal{P}(\arch) \rightarrow Ob(\euc^{\fr}).
  \]
  
  More precisely, its values in $Ob(\euc^{\fr})$ are models $\fr
  \xrightarrow{\model_i} \euc$.
  The map $\mathcal{M}_{\euc}$ map naturally extends to $\mathcal{M}_{\set}:
  \ps(\arch) \rightarrow \set^{\fr}$ via the forgetful functor $\euc
  \rightarrow \set$. When considering $\mathcal{M}_{\set}$ we will simply write $\mathcal{M}$.
\end{definition}

\begin{example}\label{ex:param_function_special_case}
  We show how the parameter-function map arises as a special case of the parameter-functor map.
  Let $\fr =$ \begin{tikzcd}[column sep=small, ampersand replacement=\&]
  \LTO{A}\ar[r, "f"]\& \LTO{B} \\
\end{tikzcd}. Consider the architecture $\begin{tikzcd}[column sep=10, row sep=14, ampersand replacement=\&] \LTO{A}\ar[r, "f"]\& \LTO{B} \\
\end{tikzcd} \xrightarrow{\arch} \parae$ such that $\arch(f) :
\RR^n \times \RR^a \rightarrow \RR^b$. Then $\mathcal{P}(\arch) = \RR^n$.
The parameter-functor map partially applied to $\arch$ has the following type $\pspec(\arch, -) :
\RR^n \rightarrow Ob(\euc^{\begin{tikzcd}[column sep=10, row sep=3, ampersand replacement=\&] \LTO{A}\ar[r, "f"]\& \LTO{B} \\
\end{tikzcd}})$. For each $p \in \RR^p$ it specifies a choice of a functor
which can be identified with a function $\arch(f)(p, -) : \RR^a \rightarrow
\RR^b$.
This means $\pspec(\arch, -)$ can be identified with a function $\RR^n
\rightarrow \RR^{b^{\RR^a}}$, thus reducing to the parameter-function map.
\end{example}

Parameter-functor map $\mathcal{M} : \ps(\arch) \rightarrow Ob(\set^{\fr})$ is an important construction because of the following reason.
It allows us to consider its image in $\set^{\fr}$ as a \textit{space} which
inherits all the properties of $\euc$.
We explore this statement in detail in Section \ref{sec:functor_space}.

\section{Data}

We have described constructions which allow us to pick an architecture for a
schema and consider its different models $\model_p$, each of them identified with a
choice of a parameter $p \in \ps(\arch)$.
In order to understand how the optimization process is steered in
updating the parameter choice for an architecture, we need to understand
a vital component of any deep learning system -- datasets themselves. 

Understanding how datasets fit into the broader picture necessitates that we also understand their relationship to the space they are embedded in.
Recall the \textit{embedding} functor and the notation $\ccv$ for the discretizaton of some category $\cc$.

\begin{definition}\label{def:datasetf}
Let $\discr{\fr} \xrightarrow{E} \set$ be the embedding. A
\textbf{dataset} is a subfunctor $\Fun{D}_E: \vfr \rightarrow \set$ of $E$.
$D_E$ maps each object $A \in Ob(\fr)$ to a dataset $\Fun{D}_E(A) :=
\{a_i\}_{i=1}^N \subseteq \Fun{E}(A)$.
\end{definition}

Note that we refer to this functor in the singular, although it assigns a
dataset to \textit{each} object in $\fr$. We also highlight that the domain of
$\Fun{D}_E$ is $\vfr$, rather than $\fr$. We generally cannot provide an action on
morphisms because datasets might be incomplete.
Going back to the example with Horses and Zebras -- a dataset functor on $\fr$ in
Figure \ref{fig:birdseye} (b) maps $\mathrm{Horse}$ to the set of obtained horse images
and $\mathrm{Zebra}$ to the set of obtained zebra images.

The subobject relation $D_E \subseteq E$ in Proposition \ref{def:datasetf} reflects an important property of data;
we cannot obtain some data without it being in some shape or form, embedded in
some larger space. Any obtained data thus implicitly fixes an embedding.

Observe that when we have a dataset in standard machine learning, we have a
dataset \textit{of something}. We can have a dataset of historical weather data,
a dataset of housing prices in New York or a dataset of cat images.
What ties all these concepts together is that each element $a_i$ of some
dataset $\{a_i\}_{i=1}^N$ is an instance of a more general concept. As a trivial
example, every image in the dataset of horse images is a
\textit{horse}. The word \textit{horse} refers to a more general concept and as
such could be generalized from some of its instances which we \textit{do not
  possess}. But all the horse images we possess are indeed an example of a
horse.
By considering everything to be embedded in some space $E(A)$ we capture this
statement with the relation $\{a_i\}_{i=1}^N \subseteq \cpt(A) \subseteq
E(A)$.
Here $\cpt(A)$ is the set of all instances of some notion $A$ which are
embedded in $E(A)$. In the running example this corresponds to all images of
horses in a given space, such as the space of all $64 \times 64$ \texttt{RGB}
images.
Obviously, the precise specification of $\cpt(A)$ is unknown -- as we cannot
enumerate or specify the set of \textit{all} horse images.

We use such calligraphy to denote this is an abstract concept.
Despite the fact that its precise specification is unknown, we can still reason about its relationship to other structures.
Furthermore, as it is the case with any abstract notion, there might be some
edge cases or it might turn out that this concept is ambiguously defined or
even inconsistent. Moreover, it might be possible to identify a dataset with
multiple concepts; is a dataset of male human faces associated with the concept of
male faces or with with all faces in general?
We ignore these concerns and assume each dataset is a dataset of some
well-defined, consistent and unambiguous concept. This does not change the
validity of the rest of the formalism in any way as there exist plenty of
datasets satisfying such a constraint.

Armed with intuition, we show this admits a generalization to the
categorical setting.
Just as $\{a_i\}_{i=1}^N \subseteq \cpt(A) \subseteq
E(A)$ are all subsets of $E(A)$ we might hypothesize $D_E \subseteq \cpt 
\subseteq E$ are all subfunctors of $E$. This is quite close to being true. It
would mean that the domain of $\cpt$ is the discrete category $\vfr$. However,
just as we assign a set of all concept instances to \textit{objects} in $\fr$,
we also assign a function between these sets to \textit{morphisms} in $\fr$. Unlike with
datasets, this can be done because, by definition, these sets are not incomplete.

We make this precise as follows. Recall the inclusion functor ${\vfr \xhookrightarrow{I} \qfr}$.

\begin{definition}
  Given a schema $\qfr$ and a dataset $\vfr \xrightarrow{D_E} \set$, a
  \textbf{concept} associated with the dataset $D_E$ embedded in $E$ is a
  functor $\cpt: \qfr \rightarrow \set$ such that $D_E \subseteq \cpt \circ I
  \subseteq E$. We say $\cpt$ picks out sets of concept instances and functions
  between those sets. 
\end{definition}

Another way to understand a concept $\qfr \xrightarrow{\cpt} \set$ is that it is
required that a human observer can tell, for each $A \in Ob(\fr)$
and some $a \in \emb(A)$ whether $a \in \cpt(A)$. Similarly for morphisms, a
human observer should be able to tell if some function $\cpt(A) \xrightarrow{f}
\cpt(B)$ is an image of some morphism in $\qfr$ under $\cpt$.

For instance, consider the GAN schema in Figure \ref{fig:birdseye} (a) where
$\cpt(\mathrm{Image})$ is a set of all images of human faces embedded in some
space such as $\RR^{64 \times 64 \times 3}$. For each image in this space, a human observer should be able to tell if that image contains a face or not.
We cannot enumerate such a set $\cpt(\mathrm{Image})$ or write it down
explicitly, but we can easily tell if an image contains a given concept.
Likewise, for a morphism in the CycleGAN schema (Figure \ref{fig:birdseye} (b)),
we cannot explicitly write down a function which transforms a horse into a
zebra, but we can tell if some function did a good job or not by testing it on
different inputs.

The most important thing related to this concept is that this represents the
goal of our optimization process. Given a dataset $\vfr \xrightarrow{D_E} \set$,
want to extend it into a functor $\qfr \xrightarrow{\cpt} \set$, and actually
\textit{learn} its implementation.

\section{Optimization}

\sloppy
We now describe how data guides the search process. We identify
the goal of this search with the concept functor $\qfr
\xrightarrow{\cpt} \set$. This means that given a schema $\qfr$ and data
$\vfr \xrightarrow{D_E} \set$ we want to train some architecture $\arch$ and
find a functor ${\qfr \xrightarrow{H'} \set}$ that can be identified with $\cpt$. Of
course, unlike in the case of the concept $\cpt$, the implementation of $H'$ is
something that will be known to us.

We now define the notion of a \textit{task}.

\begin{definition}\label{def:task}
Let $G$ be a directed multigraph and $\sim$ a congruence relation on $\fr$.
A task is a triple $(G, \sim, \vfr \xrightarrow{D_E} \set)$.
\end{definition}

In other words, a graph $G$ and $\sim$ specify a schema $\qfr$ and a functor $D_E$
specifies a dataset for that schema.
Each dataset is a dataset \textit{of something} and thus can be associated with a
functor $\qfr \xrightarrow{\cpt} \set$.
Moreover, recall that a dataset fixes an embedding $E$ too, as $D_E \subseteq E$.
This in turn also narrows our choice of architecture $\fr \xrightarrow{\arch}
\parae$, as it has to agree with the embedding on objects. This situation fully
reflects what happens in standard machine learning practice -- a neural network
$P \times A \rightarrow B$ has to be defined in such a way that its domain $A$
and codomain $B$ embed the datasets of all of its inputs and outputs, respectively.

Even though for the same schema $\qfr$ we might want to consider different
datasets, we will always assume a chosen dataset corresponds to a single training goal $\cpt$.

\subsection{Optimization Objectives}

As it is the general theme of this thesis -- we generalize an
already established construction to a categorical setting in a natural way, free
of ad-hoc choices. This time we focus on the training procedure as described in \cite{CycleGAN}.

Suppose we have a task $(G, \sim, \vfr \xrightarrow{D_E} \set)$.
After choosing an architecture $\fr \xrightarrow{\arch}
\parae$ consistent with the embedding $E$ and, hopefully, with the right inductive bias, we start with a randomly chosen parameter
$\theta_0 \in \ps(\arch)$. Via the parameter-functor map (Definition
\ref{def:param_functor_map}), this amounts to the choice of a specific $\fr \xrightarrow{\model_{\theta_0}} \euc$.
Using the loss function defined further down in this section, we partially
differentiate each $f : \RR^n \times \RR^a \rightarrow \RR^b \in \gen{\fr}$ with
respect to the corresponding $p_f$ . We then obtain a new parameter value for
that function using some update rule, such as Adam \citep{Adam}.
The product of these parameters for each of the generators $(p_f)_{f \in
  \gen{\fr}}$ (Definition \ref{def:parameter_space}) defines a new parameter
$\theta_1 \in \ps(\arch)$ for the model $\model_{\theta_1}$.
This procedure allows us to iteratively update a given $\model_{\theta_i}$ and
as such fixes a linear order $\{\theta_0, \theta_1, \dots, \theta_T\}$ on some subset of $\ps(\arch)$.

\sloppy
The optimization objective for a model $\fr \xrightarrow{\model_{\theta}} \euc$ and a task $\task$ is twofold. The total loss will be stated as
a sum of the \textit{adversarial loss} and a \textit{path-equivalence loss}. We
now describe both of these losses.

Adversarial loss is tightly linked to an important, but orthogonal concept to
this thesis -- Generative Adversarial Networks (GANs). As we slowly transition to standard machine learning terminology, we note that some of the notation here will be untyped due to the lack of the proper categorical understanding of these concepts.\footnote{
  In this thesis this adversarial component is used in the optimization
  procedure, but to the best of our knowledge, it has not been properly framed in
  a categorical setting yet and is still an open problem. It seems to require
  nontrivial reformulations of existing constructions \citep{BackpropAsFunctor}
  and at least a partial integration of Open Games \citep{CompositionalGameTheory} into the framework of gradient-based optimization. In this thesis we do not concern
  ourselves with these matters and they present no practical issues for training
  these networks.}

We now roughly describe how discriminators fit into the story so far, assuming a
fixed architecture.
We assign a discriminator to each object $A \in Ob(\fr)$ using the following function:
\begin{equation}
\mathbf{D} : (A : Ob(\fr)) \rightarrow \parae(\arch(A), \RR)
\label{eq:discr_function}
\end{equation}

This function assigns to each object $A \in Ob(\fr)$ a morphisms in $\parae$ such that its
domain is that given by $\arch(A)$. This will allow us to compose compatible
generators and discriminators.
For instance, consider $\arch(A) = \RR^a$. Discriminator $\bd(A)$ is then a
function of type $\RR^q \times \RR^a \rightarrow \RR : \parae(\RR^a, \RR)$,
where $\RR^q$ is discriminator's parameter space.
As a slight abuse of notation -- and to be more in line with machine learning
notation -- we will call $\bd_A$ discriminator of the object $A$ with some partially applied parameter value $\bd(A)(p, -)$.

In the context of GANs, when we refer to a generator we refer to the image of a generating morphism in $\fr$ under $\arch$.
Similarly as with discriminators, a generator corresponding
to a morphism $\RR^a \xrightarrow{f} \RR^b$ in $\parae$ with some partially applied
parameter value will be denoted using $\bg_f$.

The GAN minimax objective $\mathcal{L}_{GAN}^B$ for a generator $\bg_f$ and a
discriminator $\bd_B$ is stated in Eq. \eqref{eq:wgan-gp}. In
this formulation we use Wasserstein distance \citep{WGAN}.

\begin{equation}
  \begin{split}
    \mathcal{L}_{GAN}^B(\bg_f, \bd_B) &:= \mathop{\mathbb{E}}_{b \sim D_E(B)} \left[ \bd_B(b) \right] \\
    &- \mathop{\mathbb{E}}_{a \sim D_E(A)} \left[ \bd_B(\bg_f(a))  \right]
  \end{split}
\label{eq:wgan-gp}
\end{equation}

The generator is trained to minimize the loss in the 
Eq. \eqref{eq:wgan-gp}, while the discriminator is trained to maximize it.

The second component of the total loss is a generalization of
\textit{cycle-consistency loss} in CycleGAN \citep{CycleGAN}, analogous to the
generalization of the cycle-consistency condition in Section \ref{sec:model_schema}.

\begin{definition}
 Let $A \doublerightarrow{f}{g} B$ and suppose there exists a path equivalence
 $f = g$. For the equivalence $f = g$ and the model $\fr
 \xrightarrow{\mathsf{Model}_i} \euc$ we define a \textbf{path equivalence loss} $\mathcal{L}_{\sim}^{f, g}$   as follows:
 \[
 \mathcal{L}_{\sim}^{f, g} := \mathbb{E}_{a \sim
    D_E(A)} \big[ {\vert \vert \model_i(f)(a) - \model_i(g)(a) \vert \vert}_1
  \big]
  \]
\end{definition}

This enables us to state the total loss simply as a weighted sum of
adversarial losses for all generators and path equivalence losses for all
equations.

\begin{definition}\label{def:total_loss}
The \textbf{total loss} is given as the sum of all adversarial and path
equivalence losses:
\begin{align}
  \label{eq:total_loss}
  \mathcal{L}_i &:= \sum_{A \xrightarrow{f} B \in \gen{\fr}}
 {\mathcal{L}_{GAN}^B(\bg_f, \bd_B)} + \gamma\sum_{f = g
  \in \sim} \mathcal{L}_{\sim}^{f, g}
\end{align}
where $\gamma$ is a hyperparameter that balances between the adversarial loss
and the path equivalence loss.
\end{definition}

Observe that identity morphisms are not elements of $\gen{\fr}$. Even if they
were, they would cause the discriminator to be steered in two directions. They
would incentivize the discriminator to classify samples of $D_E(A)$ as real, but also to classify samples of $D_E(A)$ under $\bg_{id_A}: A \rightarrow A$ as fake. But notice that
$\bg_{id_A}$ does not change the samples, so it would be pushed to classify the
same samples as both real and fake -- making the net result is zero as these
effects cancel each other out.

\subsection{Restriction of Network Instance to the Dataset}

We have seen how data relates to the architecture and how model $\model_{p_i}$
corresponding to a parameter $p_i \in \ps(\arch)$ is updated.
Observe that network instance $H_p$ maps each object $A \in Ob(\fr)$ to the
entire embedding $H_{p_i}(A) = E(A)$, rather than just the concept $\cpt(A)$. Even though we started out with an embedding $E(A)$, we want to narrow that embedding down just to the set of instances corresponding to some concept $A$.

For example, consider a diagram such as the one in Figure \ref{fig:birdseye} (a). Suppose the result of a successful training was a functor $\fr \xrightarrow{H} \set$.  Suppose that the image of $h$ is $H(h) : \zo \rightarrow \gimg$. As such, our interest is mainly the
restriction of $\gimg$ to $\cpt(\mathrm{Image})$, the image of $\zo$ under $H(h)$, rather than the
entire $\gimg$. In the case of horses and zebras in Figure \ref{fig:birdseye} (b), we are interested in a map
$\cpt(\mathrm{Horse}) \rightarrow \cpt(\mathrm{Zebra})$ rather than a map $\gimg
\rightarrow \gimg$.
In what follows we show a construction which restricts some $H_p$ to its
smallest subfunctor which contains the dataset $D_E$.
Recall the previously defined forgetful functor $\euc \xrightarrow{U} \set$ and the
inclusion $\vfr \xhookrightarrow{I} \fr$.

\begin{definition}
Let $\Fun{D}_E: \vfr \rightarrow \set$ be the \textbf{dataset}. Let $\fr
\xrightarrow{H_p} \set$ be the network instance on $\fr$. The \textbf{restriction}
of $H_p$ to $D_E$ is a subfunctor of $H_p$ defined as follows:
\[
\ih := \bigcap\limits_{\{G \in
    Sub(H_p)) \mid \Fun{D}_E \subseteq \Fun{G} \circ \Fun{I}\}}G
\]
where $Sub(H_p)$ is the set of subfunctors of $H_p$.
\end{definition}

This definition is quite condensed so we supply some intuition.
We first note that the meet is well defined because each $G$ is a subfunctor of $H$.

In Figure \ref{fig:data_functors} we depict the newly defined constructions
using a commutative diagram.

\tikzset{
    labl2/.style={anchor=south, rotate=-30, inner sep=.5mm}
}
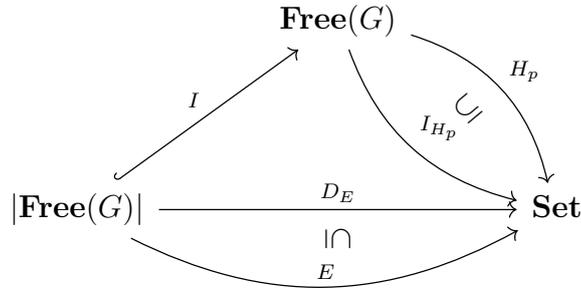
\begin{figure}[H]
\centering
\begin{tikzcd}[column sep=40pt, row sep=50pt]
  & \fr \arrow[dr, "\mathbin{\rotatebox[origin=c]{40}{$\subseteq$}}", phantom,
  bend left=10] \arrow[dr, "H_p", bend left] \arrow[dr, "\ih", bend right=28] & \\
  \vfr \arrow[ru, hook, "I"] \arrow[rr, phantom, "\mathbin{\rotatebox[origin=c]{-90}{$\subseteq$}}", bend right=10] \arrow[rr, "D_E"] \arrow[rr, "E", bend right=28] & & \set
\end{tikzcd}
\caption{The functor $\ih$ is a subfunctor of $H_p$, and $D_E$ is a
  subfunctor of $\ih \circ I$.}
\label{fig:data_functors}
\end{figure}

It is useful to think of $I_H$ as a restriction of $\Fun{H}$ to the \textit{smallest} functor which fits all data and mappings between the data.
This means that $\ih$ contains all data samples specified by $D_E$.
\begin{corollary}
$D_E$ is a subfunctor of $\ih \circ I$:
\end{corollary}
\begin{proof}
This is straightforward to show, as $\ih$ is the intersection of all subobjects
of $H$ which, when composed with the inclusion $I$ contain $D_E$. Therefore $\ih
\circ I$ contains $D_E$ as well.
\end{proof}

Even though Corollary \ref{cor:action_on_objects} tells us that all $H_p$ act
the same on objects, we can see that this is not the case with $\ih$.

Consider some $A \xrightarrow{f} B$ in $\fr$. Suppose we have two two different
network instances, $\RR^a \xrightarrow{H_p(f)} \RR^b$ and $\RR^a
\xrightarrow{H_q(f)} \RR^b$. Even though these instances act the same on
objects, their restrictions to data $D_E$ might not. Depending on the
implementations of $H_p(f)$ and $H_q(f)$, images of $\RR^a$ under $H_p(f)$ and
$H_q(f)$ might be different, and as such impose different constraints on the
minimum size of their restriction to $D_E$ on the object $B$. This in
turn means that the sets $I_{H_p}(B)$ and $I_{H_q}(B)$ might be vastly different
and practically agree only on the $D_E(B)$.
More generally, for any $X \in Ob(\fr)$ and $X \xrightarrow{g} A$, $\ih(A)$ contains the image of $D_E(X)$ under $H(X)
\xrightarrow{H(g)} H(A)$.

% \[
% \ih(A) := \bigcup\limits_{X \in Ob(\fr), g \in \fr(X, A)}{H_p(g)[D_E(X)]}}
% \]

\subsection{Path Equivalence Relations}\label{sec:path_equiv}

There is one interesting case of the total loss -- when the total
path-equivalence loss is zero: $\sum_{f = g \in \sim}
\mathcal{L}_{\sim}^{f, g} = 0$. This tells us that $H(f) =
H(g)$ for all $f = g$ in $\sim$.
In what follows we elaborate on what this means by recalling how $\fr$ and
$\qfr$ are related.

So far, we have been only considering schemas given by $\fr$. This indeed is a limiting
factor, as it assumes the categories of interest are only those without any
imposed relations $R$ between the generators $G$. One example of a schema
\textit{with} relations is the CycleGAN schema \ref{fig:birdseye} (b).
As briefly mentioned in Remark \ref{rem:relations}, fixing a functor $\qfr
\rightarrow \set$ requires its image to satisfy any relations imposed by $\qfr$.
As neural network parameters usually are initialized randomly, any such image in
$\set$ will most surely not preserve such relations and thus will not be a
proper functor whose domain is $\qfr$.

However, this construction is a functor if we consider its domain to be $\fr$.
Furthermore, assuming a successful training process whose end result is a
path-equivalence relation preserving functor $\fr \rightarrow \set$, we show this induces an unique $\qfr \rightarrow \set$ (Figure \ref{fig:quot_cat}).

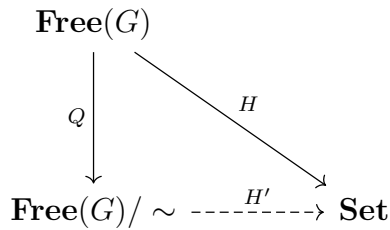
\begin{figure}[H]
\centering
\begin{tikzcd}[column sep = 50pt, row sep = 50pt]
\free(G) \arrow[rd, "\Fun{H}"] \arrow[d, "\Fun{Q}"'] &  \\
\qfr \arrow[r, dashed, "\Fun{H'}"]
& \Cat{Set}
\end{tikzcd}
\caption{Functor $H$ which preserves path-equivalence relations factors uniquely
  through $Q$.}
\label{fig:quot_cat}
\end{figure}

\begin{theorem}\label{thm:quot_theorem}
Let $\Fun{Q}: \fr \rightarrow \qfr$ be the quotient
functor and let $\Fun{H}: \fr \rightarrow \set$ be a path-equivalence relation
preserving functor. Then there exists a unique functor $\Fun{H'}: \qfr \rightarrow \set$ such that $\Fun{H'}\circ \Fun{Q} = \Fun{H}$.
\end{theorem}

\begin{proof}
We refer the reader to \cite{WorkingMathematician}, Section 2.8., Proposition 1.
\end{proof}

Finding such a functor $\Fun{H}$ is no easier task than finding a functor
$\Fun{H'}$. However, considering the domain to be $\fr$, rather than $\qfr$
allows us to initially just guess a functor $\Fun{H_0}$, since our initial
choice will not have to preserve the equivalence $\sim$.
As training progresses and the path-equivalence loss of a network instance $\fr
\xrightarrow{H_p} \set$ converges to zero, by Theorem \ref{thm:quot_theorem} we
show $H_p$ factors uniquely through $\fr \xrightarrow{Q} \qfr$ via $\qfr
\xrightarrow{H'} \set$.

\subsection{Functor Space}\label{sec:functor_space}

Recall the parameter-functor map (Definition \ref{def:param_functor_map})
$\mathcal{M}_{\euc}: \ps(\arch) \rightarrow Ob(\euc^{\fr})$.
Each choice of $p \in \ps(\arch)$ specifies a functor of type $\fr
\rightarrow \set$.
In this way exploration of the parameter space amounts to exploration of part
of the functor category $\set^{\fr}$.
Roughly stated, this means that a choice of an architecture $\fr \xrightarrow{\arch}
\parae$ adjoins a notion of \textit{space} to the image of $\ps(\arch)$ under
$\mathcal{M}_{euc}$ in the functor category $\euc^{\fr}$. This space inherits
all the properties of $\euc$.

By using gradient information to search the parameter space $\ps(\arch)$, we
are effectively using gradient information to search part of the functor space
$\set^{\fr}$. Although we cannot explicitly explore just $\set^{\qfr}$,
we penalize the search method for veering into the parts of this space where the
specified path equivalences do not hold. As such, the inductive bias of the
model is increased without special constraints on the datasets or the embedding space --
we merely require that the space is differentiable and that is has a sensible notion of
distance.
Note that we do not claim inductive bias is \textit{sufficient} to guarantee
training convergence, merely that it is a useful regularization method applicable to
a wide variety of situations. 

As categories can encode complex relationships between concepts and as functors map
between categories in a structure-preserving way -- this enables
\textit{structured learning} of concepts and their interconnections in a very
general fashion. Of course, this does not tell us which such structures are
practical to learn or what other inductive biases we need to add, but simply
disentangles those structures such that they can be studied in their own right.

To the best of our knowledge, this represents a considerably different approach to
learning. Rather than employing a gradient-based search in \textit{one} function
space $Y^X$ this formulation describes a search method in a \textit{collection} of
such function spaces and then regularizes it to satisfy any composition rules
that have been specified.
Even though it is just a rough sketch, this concludes our reasoning which shows functors of the type $\fr \rightarrow \set$ can be \textit{learned} using gradient-descent.

\section{From Categorical Databases to Deep Learning}\label{sec:fdm_correspondence}

% Before moving on to examples, let us pause and summarize the story so far.

The formulation presented in this thesis bears a striking and unexpected
similarity to Functorial Data Migration (FDM) \citep{FunctorialDataMigration}.
Given a \textit{categorical schema} $\cc$ on some directed multigraph $G$, FDM specifies a functor
category $\Cat{Set}^{\cc}$ of database instances on that schema. The notion of \textit{data integrity} is captured by \textit{path equivalence
  relations} which constrain the schema by demanding specific paths in the graph
be the same.
The analogue of data integrity is captured by cycle-consistency conditions,
first introduced in CycleGAN \citep{CycleGAN}.
In Table \ref{tab:rosetta} we outline the main differences between the approaches.
Namely, our approach requires us not to specify functors into $\set$, but rather
to \textit{learn} them.

This shows that the underlying structures used for specifying data semantics and
concrete instances for a given database system are equivalent to the structures
used for describing network semantics and their concrete instances.

\begin{table*}
  \centering
  \begin{center}
    \begin{tabular}{| l | c | c |}
      \hline
      & $\qfr \rightarrow \set $ &  $\fr \rightarrow \qfr$ \\ \hline
      Functorial Data Migration & Fixed  & Data integrity \\ \hline
      Compositional Deep Learning & Learned & Cycle-consistency \\ \hline
    \end{tabular}
  \end{center}
  \caption{Pocket version of the Rosetta stone between Functorial Data Migration and Compositional Deep Learning}
  \label{tab:rosetta}
\end{table*}

\chapter{Examples}\label{ch:examples}

We have seen how specific concepts in deep learning fit into a rigorous
categorical framework. We shed further light on this framework by considering
some of its interesting special cases. Namely, we show how different instances
of a task $\task$ give rise not just to two existing architectures, but also to a completely new system.

\section{Existing Architectures}

We first shift our attention to two existing architectures: GAN and CycleGAN.
We specify the choice of $G, \sim$ and the dataset $\fr \xrightarrow{D_E} \set$
and show this determines our interpretation of the learned semantics.

\subsection{GAN Task}

The choice of a graph and path equivalence relations
corresponding to the GAN task can be identified with the schema in Figure
\ref{fig:birdseye} (a) and a choice of a dataset.
We will abbreviate ``Latent space'' with $LS$ and ``Image space'' with
$IS$.
A choice of a dataset can be understood as follows.
We think of  $D_E(LS) = \zo$ as the choice of uniform
distribution in some high-dimensional space and $D_E(IS) \subseteq \gimg$ as the
choice of some image dataset, such as faces, cars, or handwritten digits.
The choice of a dataset fixes the embedding $\vfr \xrightarrow{E} \set$, but also
the choice of a concept $\cpt$. In our case $\cpt(IS)$ is the set of $64 \times
64$ \texttt{RGB} images human faces and $\cpt(LS)$ is some notion of an indexing
set. As the corresponding dataset $D_E(IS)$ does not have any meaningful
interpretation, neither does $\cpt(IS)$.

There are no path equivalences here, so the total loss reduces just to the
adversarial loss. Moreover, since the only generator in this schema is $LS
\xrightarrow{h} IS$, the total adversarial loss reduces just to the standard
minimax loss.
As per Example \ref{ex:param_function_special_case}, the parameter space of any
chosen model $\fr \xrightarrow{\model} \parae$ reduces to the parameter space of
the function $\arch(h)$.

In a similar fashion, the network instance $H_p$ can be identified with a choice
of a function $H_p(h)$. Observe that the image some function $\zo
\xrightarrow{H_p(h)} \sgimg$ might only partially overlap the dataset $D_E(IS)$.
Nevertheless, the induced $\ih$ allows us to consider the union of the dataset
and the image of this function.

\subsection{CycleGAN Task}

CycleGAN task corresponds to the choice of schema in Figure
\ref{fig:birdseye} (b) and a choice of two datasets for each of the objects whose corresponding concepts are isomorphic.
As previously mentioned, these could be sets of images of \textit{apples and oranges},
\textit{paintings and photos}, \textit{horses and zebras} etc.

The parameter space of some chosen architecture $\fr \xrightarrow{\arch} \parae$
for this schema is the product $\ps(\arch) = \arch(f) \times \arch(g)$.
There are two discriminators -- one which distinguishes between real and fake
horse images and the one which distinguishes between real and fake zebra images.
Unlike in the GAN task, the total loss does not reduce just to the adversarial
one. Furthermore, the adversarial component of the loss is the sum of \textit{two}
minimax losses, for $(\bg_f, \bd_B)$ and $(\bg_g, \bd_A)$.

\section{Product Task}\label{sec:product_schema}

We have given two examples of a categorical \textit{schema}.
Each of them has specific semantics which can be interpreted. CycleGAN posits two
objects $A$ and $B$ are \textit{the same} in a certain way by learning an
isomorphism $A \cong B$. GAN has a simple semantic which tells us that one
object $A$ indexes at least part of another object $B$ in some way -- i.e.~that there exists a mapping $A \rightarrow B$. Together
with the choice of a dataset, they constitute a choice of a task.

We now present a different choice of a dataset for the CycleGAN schema which
makes up a task we will call \textit{the product task}. 
The interpretation of this task comes in two flavors. From one perspective, it
is a simple change of dataset for the CycleGAN schema. The other
perspective justifies the name \textit{product task} by grounding the
explanation in categorical terms.

Just as we can take the product of two real numbers $a, b \in \RR$ with a
multiplication function $(a, b) \mapsto ab$, we show we can take a product of
some two sets of images $A, B \in Ob(\set)$ with a neural network of type $A \times B
\rightarrow C$. We will show $C \in Ob(\set)$ is a set of images which possesses
all the properties of a categorical product. 

In a Cartesian category such as $\set$ there already exists a notion of a
categorical product -- the Cartesian product. Recall that it is unique up to unique isomorphism.
This means that any other object $X$ which satisfies the universal property of
the categorical product of $A$ and $B$ is uniquely isomorphic to $A \times B$.
This isomorphism will be central to the notion of the product task.
Before continuing, we supply some intuition with Figure
\ref{fig:product_intuition}.

\begin{figure}[H]
\centering
\begin{tikzcd}[column sep = 30pt, row sep = 50pt]
& {\includegraphics[width=0.06\textwidth]{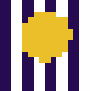}} \arrow[dl, "\theta_A"] \arrow[dr, "\theta_B"]\arrow[rrrr, dashed, bend
left=20, "d"] \arrow[rrrr, phantom, "\cong"] & & & & {\includegraphics[width=0.06\textwidth]{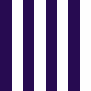}} \raisebox{9pt}{$\times$} {\includegraphics[width=0.06\textwidth]{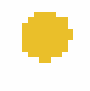}} \arrow[dl, "\pi_A"]
\arrow[dr, "\pi_B"] \arrow[llll, dashed, bend left=20, "c"] & \\
{\includegraphics[width=0.06\textwidth]{new_circles/stripes}}  & & {\includegraphics[width=0.06\textwidth]{new_circles/circle}} &  & {\includegraphics[width=0.06\textwidth]{new_circles/stripes}} & & {\includegraphics[width=0.06\textwidth]{new_circles/circle}}
\end{tikzcd}
\caption{Two different notions of a categorical product of $A$ and $B$ in $\set$,
  illustrated with examples from some choice of those sets. Right
  side depicts the Cartesian product $A \times B$, while the left side depicts
  some different product which we call $AB$. Both of them have a notion of
  projections to their constituents. By the universal property
  of the categorical product, they are uniquely isomorphic. This central
  isomorphism represents the CycleGAN perspective of the product task. The
  decomposition map $d: AB \rightarrow A \times B$ and the composition map $c: A
  \times B \rightarrow AB$ are things we want to learn.}
\label{fig:product_intuition}
\end{figure}

By \textit{learning} the model $\fr \xrightarrow{\model} \euc$
corresponding to the isomorphism $AB \cong A \times B$
we are also learning the projection maps $\theta_A$ and $\theta_B$. This follows
from the universal property of the categorical product: $\pi_A \circ d =
\theta_A$ and $\pi_B \circ d = \theta_B$.
Note how $AB$ differs from a Cartesian product. The domain of the
corresponding projections $\theta_A$ and $\theta_B$ is not a simple pair of
objects $(a, b)$ and thus these projections cannot merely discard an element.
$\theta_A$ needs to learn to remove $A$ from a potentially complex domain. As
such, this can be any complex, highly non-linear function which satisfies
coherence conditions of a categorical product.

We will be concerned with supplying this new notion of the product $AB$ with a
dataset and learning the image of the isomorphism $AB \cong A
\times B$ under $\fr \xrightarrow{\model} \set$. We
illustrate this on a concrete example. Consider a dataset $A$ of images of human
faces, a dataset $B$ of images of sunglasses, and a dataset $AB$ of people
\textit{wearing} glasses. Learning this isomorphism amounts to learning two
things: (i) learning how to decompose an image of a person wearing glasses $(ab)_i$ into an image of
a person $a_j$ and image $b_k$ of these glasses, and (ii) learning how to map this person $a_j$ and perhaps some other glasses $b_l$ into an image of a
person $a_j$ wearing glasses $b_l$. Generally, $AB$ represents some sort of composition of
objects $A$ and $B$ in the image space such that all information about $A$ and
$B$ is preserved in $AB$.
Of course, this might only be approximately true. Glasses usually cover a part
of a face and sometimes its dark shades cover up the eyes -- thus losing
information about the eye color in the image and rendering the isomorphism
invalid. However, in this thesis we ignore such issues and assume
that the networks $\arch(d)$ can learn to unambiguously fill part of the face
where the glasses were and that $\arch(c)$ can learn to generate and superimpose the glasses on the relevant part of the face. 

This example shows us that product task is tightly linked with composition and
decomposition of images.
Although we use the same CycleGAN schema from Figure
\ref{fig:birdseye} (b), we label one of its objects as $AB$ and the other one as
$A \times B$. Note that this does not change the schema itself, the labeling is merely for our convenience. We highlight that
the notion of a product or its projections is not captured in the schema itself.
As schemas are merely categories presented with generators $G$ and relations
$R$, they lack the tools needed to encode a complex abstraction such as a
universal construction. So how do we capture the notion of a product?

In this thesis we frame this simply as a specific dataset
functor, which we now describe. 
A dataset functor corresponding to the product task maps the object $A \times B$ in
CycleGAN schema to a \textit{Cartesian product of two datasets}, $D_E(A \times B) =
\{a_i\}_{i=0}^N \times \{b_j\}_{j=0}^M$. It maps the object $AB$ to a dataset
$\{(ab)_i\}_{i=0}^N$. In this case $ab$, $a$, and $b$ are free to be any
elements of datasets of a well-defined concept $\cpt$.
Although the difference between the product task and the CycleGAN task boils down to a different choice of a dataset functor, we note this is a key aspect which allows for a significantly different interpretation of the task semantics.\footnote{
  Just as we can consider a joint probability distribution of two random
  variables $X$ and $Y$, here we consider a joint \textit{dataset} of two
  datasets $\{\mathrm{a}_i\}_{i=0}^N$ and $\{b_j\}_{j=0}^M$. In a similar
  fashion we consider the Cartesian product $\bd(A) \times \bd(B)$ of two
  discriminators $\bd(A)$ and $\bd(B)$ to be a discriminator itself.}

% Just as with two previous tasks, we assume a sensible choice of datasets such
% that this isomorphism exists. 

We now describe two possible classes of choice of datasets for the product task.
One is already given in the example with glasses and faces.
This example includes three distinct datasets in its consideration. However, observe that sometimes three such related datasets might not always be available.
For instance, consider the scenario where we only possess the dataset of faces
with glasses $AB$ and a dataset of just faces $A$, but not a dataset of images
of glasses. This prevents us from generating images of glasses which
some person is wearing. But it does not prevent us from using the product task.
Namely, we know that we can decompose the point in some image space
corresponding to a person with glasses $(ab) \in D_E(AB)$ into point in some
image space of a person $a$ and of a point in some different space which
captures whatever the missing information between $(ab)$ and $a$ is. In this
case, this point would parametrize the color, shape, and type of glasses. 

Put another way, one possible choice for the action of a dataset on $A \times B$ is
$D_E(A \times B) = \{\mathrm{Faces}_i\}_{i=0}^N \times \zo$. Similarly, as a
regular GAN does not allow us to know in what way the object $LS$ indexes $IS$,
here do we not know how $\zo$, together with $D_E(B)$, indexes $D_E(AB)$.
This represents one perspective on the many-to-many mappings introduced by
\cite{AugmentedCycleGAN}. This is a map of a face to a face with glasses as a
one-to-many map.

The product schema admits a straightforward generalization to $n$ objects. For
example, the isomorphism $ABC \cong A \times B \times C$ can be understood as a
one-to-many map between $A \times B$ and $ABC$, but also as a one-to-many map
between $A$ and $ABC$.

With this in mind, we now envision potential applications of the product schema.

Consider the task of composing a car $A$ with a road $B$. In most cases there is certainly more than one way to do so -- by putting the car on different parts of this road. By generalizing the product schema to more than two datasets we begin to explore its full potential.

We could consider an isomorphism $ABC \cong A \times B \times C$, where $D_E(A)$
is a set of car images, $D_E(B)$ is a set of road images and $D_E(C) = \zo$ is
some choice of uniform distribution, similarly as with the GAN task.
Now, given a car $a \in \cpt(A)$, a road $b \in \cpt(B)$ and some parametrization the
position of car on the road $c \in \cpt(C)$, we can learn an unique composition
and decomposition function.
Observe that the purpose of $C$ is to add information on \textit{whatever is
  missing} from $A \times B$ so that the product of $A \times B$ with this
missing information is isomorphic to $ABC$.

By considering $A$ as some \textit{image background} and $B$ as the
\textit{object} which will be inserted, this allows us to interpret $d$ and $c$
as maps which \textit{remove an object from the image} and \textit{insert an
  object} in an image, respectively. They do it in a coherent, parametric way, such that
$B$ parametrizes the type of object which will be inserted into the domain $A$
and $C$ all the extra information needed to make insertion and removal inverses.
Furthermore, we restate that is all done on unpaired data. As such, no extra structure is
imposed on the \textit{sets} $A$, $B$ and $AB$, easing the data collection
process.

This seems like a novel method of object generation and deletion with unpaired
data, though a more thorough literature survey should be done to confirm this.

\chapter{Experiments}\label{ch:experiments}

In this chapter, we put on our engineering hat and test whether the product
task described in Section \ref{sec:product_schema} can be trained in practice.
We perform experiments on three datasets. The first experiment is done on a
dataset created by us and the second is done on a well-known image dataset of
celebrities.
The third experiment is significantly different and involves the usage of an
already pretrained generative model.
As such the choices of architecture and hyperparameters for first two
experiments differs from the choices for the third one.

In the first two experiments we have used optimizer Adam \citep{Adam} and the
Wasserstein GAN with gradient penalty \citep{ImprovedWGAN}. We used the
suggested choice of hyperparameters in \cite{ImprovedWGAN}.
The parameter $\gamma$ is set to $20$ and as such weighted the optimization procedure towards the path-equivalence, rather than the cycle-consistency loss.
All weights were initialized from a Gaussian distribution $\mathcal{N}(0,
0.01)$.
As suggested in \citep{ImprovedWGAN}, we always gave the discriminator a head start and trained it more, especially in the beginning. We set $n_{critic} = 50$ for the first 50 time steps and $n_{critic} = 5$ for all other time steps.

Discriminator $\bd(AB)$ and the discriminators for each $A$ and $B$ in $\bd(A
\times B)$ in first two experiments were $5$-layer ReLU convolutional neural networks
of type $\RR^q \times \RR^{32 \times 32 \times 3} \rightarrow \RR$. Kernel size
was set to $5$ and padding to $2$. We used stride $2$ to halve image size in all layers except the second, where we used stride $1$.
The only things varying in Sections \ref{sec:circles} and \ref{sec:celeba} were the filter sizes.
We used a fully-connected layer without any activations at the end of the
convolutional network to reduce the output size to $1$.

\section{Circles}\label{sec:circles}

We created datasets whose samples are shown in a suggestive manner in Figure
\ref{fig:circles_product}.

\begin{figure}[H]
\begin{center}
\begin{overpic}[width=0.5\textwidth]
  {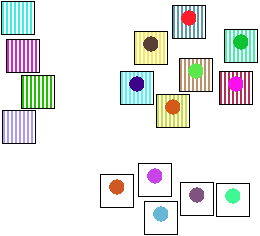}
 \put (28,60) {\Large$\xleftarrow{\theta_B}$}
 \put (65,37) {\rotatebox{-90}{$\rightarrow$}}
 \put (69,34) {$\theta_A$}
\end{overpic}
\end{center}
\caption{A dataset of stripes and circles can be decomposed sensibly into a
  dataset of circles and a dataset of stripes.}
\label{fig:circles_product}
\end{figure}

Each dataset consists of $32 \times 32$ \texttt{RGB} images of either
a circle (corresponding to $A$ in the schema), stripes (to $B$) or circle on
stripes (to $AB$), all normalized to a $[0, 1]$ range.
The background dataset $D_E(B)$ consists of unicolored stripes on a white background. 
The object dataset $D_E(A)$ consists of unicolored circles in a fixed position on a white image. The dataset $D_E(AB)$ consist of a unicolored circles
superimposed over unicolored stripes on a white background. Each element of
both $D_E(A)$ and $D_E(B)$ can be identified with a $3$-dimensional vector
describing the choice of color of the circle and the stripes, respectively.
Likewise, each element of $D_E(AB)$ can be identified with a
$6$-dimensional vector, describing the color of both the circle and the stripes.

More precisely, we define the dataset $D_E$ as follows: $D_E(A
\times B) = \{\mathrm{Circles}_i\}_{i=0}^N \times
\{\mathrm{Stripes}_j\}_{j=0}^M$ and $D_E(B) =
\{\mathrm{Circles\_on\_stripes}_k\}_{k=0}^K$.

The task of some network architecture is then to learn to decompose an image
into its two constituents. We highlight two important things. We do not provide explicit dataset
pairings to the network. We also do not tell the networks
\textit{what} the images contain. No information was specified about the circle
location or even the fact that there some kind of an object in the image.

We now specify an architecture for this task. It is given by a choice of
morphisms in $\parae$ for two generators $d
: AB \rightarrow A \times B$ and $c : A \times B \rightarrow AB$ in $\fr$.
For both of them, we choose fully-connected networks and thus flatten the images into a $3072$-dimensional vector before using them as inputs to $\model(c)$ and $\model(d)$.
This fixes the type of possible network architectures --  $\RR^p \times \RR^{3072 + 3072} \xrightarrow{\arch(c)} \RR^{3072}$ and $\RR^q \times \RR^{3072} \xrightarrow{\arch(d)} \RR^{3072 + 3072}$.

We implement both networks as simple autoencoders with a bottleneck of dimension $6$.
This means that both of them are incentivized to extract a $6$-dimensional summary of the
image -- the color of the circle and the color of the stripes.
Even though the path-equivalence relations regularize the networks to be inverses of each other, there could actually exist \textit{many} isomorphisms between the objects $\cpt(A)$ and $\cpt(B)$. For instance, the networks could learn to preserve color information by encoding it as a \textit{different} color. 
In this experiment we remedy this issue by adding an auxiliary \textit{identity mapping} loss function, as outlined in \citep{CycleGAN}. It incentivizes the training procedure to pick out a \textit{specific} isomorphism which preserves the color information in a canonical way.

% This brings the parameter count for each of those networks up to $61446$.

The architectures of the discriminators are described in the beginning of this
chapter, with a sequence of filter sizes for each layer as follows: $[3, 8, 16, 32, 64, 64]$.

\subsection{Evaluation of Trained Networks}

With this in mind, we show the results of the training.
Given that this is a toy task and that the architecture parameter count is quite low
compared to standard deep learning architectures, both networks quickly solve
the task.
The adversarial loss steers the
generators in generating realistic looking images, while the path-equivalence
loss steers the generators in preserving color information.

We test the results in three distinct ways -- by changing either $A$, $B$, or $AB$
while keeping everything else fixed and evaluating the results.
In Figure \ref{fig:change_circle} we test if the network learned to preserve
background color while the circle color is changing.
Figure \ref{fig:change_background} tests if the network learned to preserve
background color while the circle color is changing.
And lastly, Figure $\ref{fig:decompose_circle}$ test the decomposition network.

\begin{figure}[H]
\begin{center}
  {\includegraphics[width=0.6\textwidth]{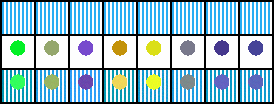}}
\caption{We evaluate the composition network by keeping the \textit{stripe color fixed} and circle color changing. First two rows are the inputs to the network and third
  row shows the results. We can see that both the stripe and circle color in the
  generated image match the expected one.}
\label{fig:change_circle}
\end{center}
\end{figure}

We observe a general pattern throughout these tests. Although the colors match
the expected ones in most cases, there exists a slight discrepancy
between the obtained and the expected results. Most of the experiments we ran exhibited similar symptoms, but we did not look further into this issue.
Nevertheless, they do manage to quickly learn and solve this simple task and as
such prove the viability of the product task.

\begin{figure}[H]
\begin{center}
  {\includegraphics[width=0.6\textwidth]{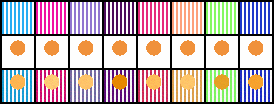}}
\caption{We evaluate the composition network by keeping the \textit{circle color fixed} and stripe color changing. First two rows are the inputs to the network and third
  row shows the results. We observe that although the stripe color matches the
  expected one, circle color varies slightly in brightness.}
\label{fig:change_background}
\end{center}
\end{figure}

\begin{figure}[H]
\begin{center}
  {\includegraphics[width=0.6\textwidth]{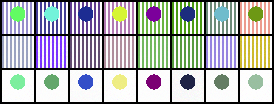}}
\caption{We evaluate the decomposition network whose input is shown in the first row. The second and third row show the generated stripe and circle
  image, respectively. We note that some discrepancy between the expected and
  generated colors exists.}
\label{fig:decompose_circle}
\end{center}
\end{figure}

\section{CelebA}\label{sec:celeba}

The circle dataset was made with the product task in mind and as such was its
canonical example. We had a notion of a dataset with circles,
a dataset with stripes and a dataset which contains both.
However, in the real world, the luxury of obtaining three such datasets related
in a special way is not always there.
We now turn our attention to an existing image dataset which we use as a more
realistic test of the feasibility of the product task.

CelebFaces Attributes Dataset (CelebA) \citep{CelebA} is a large-scale face
attributes dataset with more than 200000 celebrity images, each with 40
attribute annotations. Frequently used for image generation purposes, it fits
perfectly into the proposed paradigm of the product task.
The images in this dataset cover large pose variations and background clutter.
The attribute annotations include ``eyeglasses'', ``bangs'', ``pointy nose'',
``wavy hair'' etc., as boolean flags for every image.

We used the attribute annotations to separate CelebA into two datasets. The
dataset $D_E(AB)$ consists of images with the attribute ``Eyeglasses'', while
the dataset $D_E(A)$ consists of all the other images.

Given that we have no dataset of just glasses, we set $D_E(B) = \zo$. As outlined in section \ref{sec:product_schema}, this is a parametrization of all
the missing information from $A$ such that $A \times B \cong AB$.
We refer to an element $z \in \zo$ as a \textit{latent vector}, in line with
machine learning terminology. We will also use $Z$ instead of $B$, as to make it
more clear we are not generating images of this object.

As before, we highlight that there are no explicit pairs in this dataset -- there does not exist
an image of the same person with and without glasses in completely the same
position. Most of the people in those images appear only once and as such this
dataset provides an interesting challenge for generalization in the context of
the product task (Figure \ref{fig:faces_iso}).

\begin{figure}[H]
\vskip 0.1in
\begin{center}
  {\includegraphics[width=0.4\textwidth]{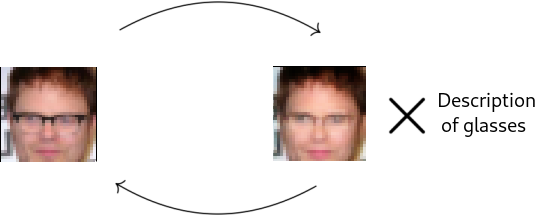}}
\caption{We trained two networks, a \textit{composition} and
  \textit{decomposition} network which learned to insert and remove glasses from
  images, respectively. We hypothesize this task of object removal and insertion
  works in more general scenarios, as nothing in the training process is specific to glasses or people.}
\label{fig:faces_iso}
\end{center}
\end{figure}

We resize all images to $32 \times 32$ to make it possible to train such network
with available compute resources. 
The network architecture of $\arch(f)$ and $\arch(g)$ are combinations of
convolutional layers and residual blocks. Unlike with the circle dataset, we do
not think of these networks as autoencoders because we have no
guarantees on the exact bottleneck size.
The architectures of the discriminators are described in the beginning of this
chapter, with a sequence of filter sizes for each layer as follows: $[3, 64, 128, 256, 512, 512]$.

The task for the decomposition neural network is then to transform an image of a
person wearing glasses into an image satisfying the following: (i) it
needs to be a realistic looking image of a person, and (ii) the identity of the
person needs to be preserved when removing glasses. The first condition depends on
the adversarial loss, while the second one depends on the path-equivalence loss.
The composition neural network has a task to transform an image of a person $a_i$ and
some parametrization of glasses $z_k$ into an image such that: (i) the image contains
a realistic looking person, (ii) the person in the image is wearing glasses,
(iii) the input of any other person $a_j$ with that vector $z$ would produce an image
of the other person wearing \textit{the same glasses}, and (iv) input of any
other glasses $z_l$ would produce the same person wearing \textit{different} glasses.
There are many different conditions here, all enforced by just two
path-equivalence relations.

We highlight that none of the networks were told that images contain
people, glasses, or objects of any kind. 

\subsection{Evaluation of Trained Networks}

Similarly as with the circle dataset, we investigate if these networks generalize
well.
We examine three things: (i) whether it is possible to
\textit{generate an image of a specific person wearing specific glasses}, (ii)
whether we can \textit{change} glasses a person wears by changing the
corresponding latent vector, and (iii) whether the same latent vector
corresponds to the same glasses, independently of the person we pair it with.

We recall that $AB$ refers to the composition, $A$ to faces and $Z$ to
parametrization of glasses.

\begin{figure}[H]
  \centering
\begin{subfigure}[t]{0.5\textwidth}
        \centering
         {\includegraphics[width=0.95\columnwidth]{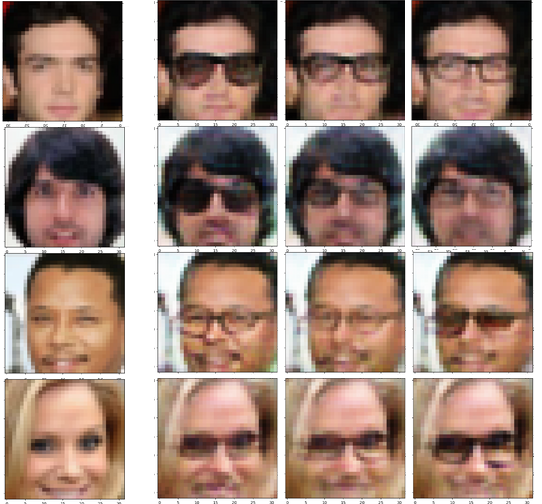}}
        \caption{}
    \end{subfigure}%
\begin{subfigure}[t]{0.5\textwidth}
        \centering
         {\includegraphics[width=0.92\columnwidth]{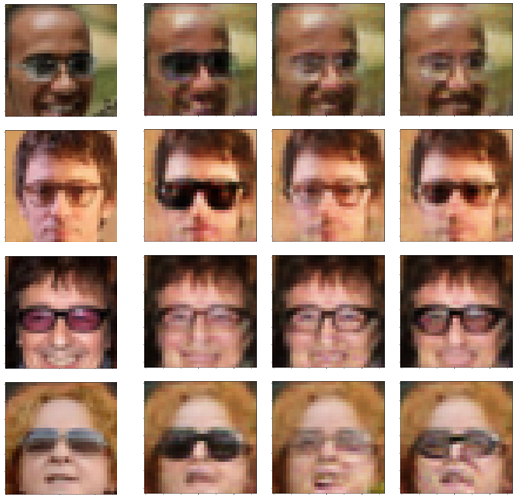}}
        \caption{}
    \end{subfigure}%
 \caption{Parametrically \textit{adding} glasses (a) and \textit{changing}
   glasses (b) on a person's face. (a): the leftmost column shows a sample
   from the dataset $a_i \in D_E(A)$. Three rightmost columns show the
   result of $c(a_i, z_j)$, where $z_j \in \zo$ is a randomly sampled latent
   vector. (b): leftmost column shows a sample from the dataset $(ab)_i \in
   D_E(AB)$. Three rightmost columns show the image $c(\pi_A(d((ab_i))), z_j)$
   which is the result of changing the glasses of a person. The latent vector
   $z_j \in \zo$ is randomly sampled.}
\label{fig:add_change_glasses}
\end{figure}

In Figure \ref{fig:add_change_glasses} (left) we show the model learns the task (i):
generating image of a specific person wearing glasses. Glasses are parametrized
by the latent vector $z \in D_E(Z)$. The model learns to warp the glasses and put them in the right angle and size, based on the shape of the face. This can especially be seen in Figure \ref{fig:z_fixed_montage}, where some of the faces are seen from an angle, but glasses still blend in naturally.
Figure \ref{fig:add_change_glasses} (right) shows the model learning task (ii):
\textit{changing} the glasses a person wears.

Figure \ref{fig:remove_glasses} shows the model can learn to \textit{remove}
glasses. Observe how in some cases the model did not learn to remove the
glasses properly, as a slight outline of glasses can be seen.

An interesting test of the learned semantics can be done by checking if a specific randomly sampled latent vector $z_j$ is consistent across different
images. Does the resulting image of the application of $g(a_i, z_j)$, contain
the same glasses as we vary the input image $a_i$?
The results for that third task (iii) are shown in Figure \ref{fig:z_fixed_montage}. It shows how the network has learned to associate a specific vector $z_j$ to a specific type of glasses and insert it in a natural way.

We note low diversity in generated glasses and a slight loss in image quality,
which is due to suboptimal architecture choice for neural networks. It
can be seen that the network has learned to preserve all the main facial
features.

\begin{figure}[H]
\begin{minipage}[t]{0.45\linewidth}
      {\includegraphics[width=\columnwidth]{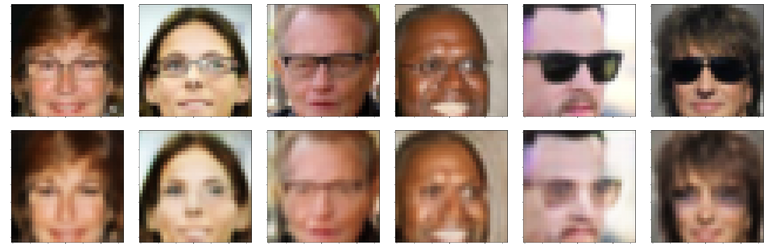}}
      \caption{Top row shows samples ${(ab)_i \in D_E(AB)}$. Bottom
        row shows the result of a function $\pi_A \circ d: AB \rightarrow A$
        which removes the glasses from the person.}
      \label{fig:remove_glasses}
\end{minipage}\qquad
\begin{minipage}[t]{0.45\linewidth}
      {\includegraphics[width=\columnwidth]{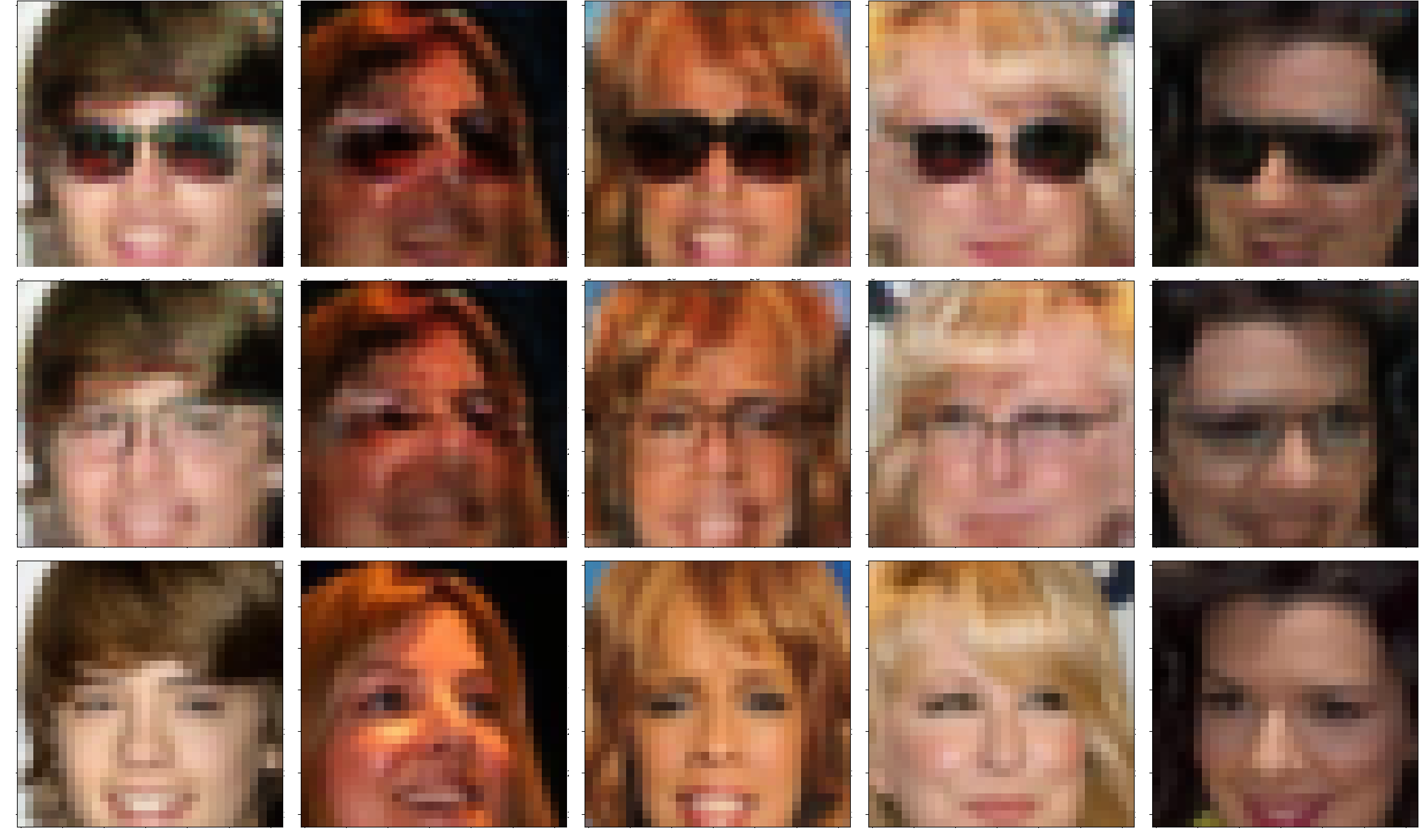}}
      \caption{Bottom row shows true samples $a_i \in D_E(A)$. Top two
        rows show the image $c(a_i, z_j)$ of adding glases with \textit{a specific latent
          vector} $z_1$ for the topmost row and $z_2$ for the middle row. Observe how the general style of the glasses stays the same in a given row, but gets adapted for every person that wears them.}
      \label{fig:z_fixed_montage}
\end{minipage}
\end{figure}

\section{StyleGAN}

In previous sections we have used a dataset $D_E$ of a concept $\cpt$ in some
\textit{image space} to train neural networks.
We now consider the same schema in Figure \ref{fig:birdseye} (b), but conceive a different notion of a dataset to test the product task on -- dataset embedded in some \textit{latent space} of a trained generator.

We turn our attention to a GAN architecture
called StyleGAN \citep{StyleGAN}. We do not attempt to train it but merely use
its pretrained generator. This generator was trained on a dataset of human
faces and is capable of producing high-resolution images of stunning quality, offering a lot of variety in terms
of pose, age, ethnicity, lightning and image background. 
The architecture of the generator -- a complex neural network with a total of $26.2$ million parameters -- will be referred to as a function ${g: \sgzo \rightarrow \sgimg}$.

Given the large variety of generated images, we turn our attention to two specific
classes of images. Observe that there exist two subsets of generated people:
those that are wearing some sort of headwear (such as glasses or a hat) and
those that are not.
Just as in Section \ref{sec:celeba}, we focus to glasses and make these statements more precise. A subset $\cpt(AB) \subseteq \sgzo$ of latent vectors are mapped to the set of images of a people \textit{wearing glasses} and another subset $\cpt(A) \subseteq \sgzo$ is mapped to a set of images of people \textit{not wearing glasses}.

In Section \ref{sec:celeba} we have shown that learning from samples in the
\textit{image space} yields promising results. In this section we hypothesize
that \textit{learning from samples in the latent space} of a trained generator
can produce similar results. As long as the dataset corresponding to the product
task is a dataset of any well-defined concept $\cpt$, we observe that the
product task is \textit{independent} of the semantic interpretation of such a
dataset. Namely, we conjecture that any structure between points in the
\textit{image space} is the image of this structure in the \textit{latent
  space}. For example, it is well-known that interpolation between two points in
the latent space yields realistic looking results when mapped to the image
space. 

As such we embark on a task of collection of data points embedded in the \textit{latent
  space}. These data points correspond to natural images of a chosen class after they
are put through the trained generator.
In other words, we collect two datasets of points which are embedded
\textit{in the latent space}. The datasets are collected as follows.
We randomly sample a vector $z$ from the latent space, put it through the trained
StyleGAN generator and display the result $g(z)$ on the screen as a $1024 \times 1024$
\texttt{RGB} image. 
By creating a script which does that for many points at once and prompts us to
classify the result as an image of a person \textit{wearing glasses} or a person
\textit{not wearing glasses}, we were able to create two such datasets,
$D_E(AB)$ and $D_E(A)$. Sizes of these datasets were $356$ and $1576$,
respectively. This class imbalance stems from the fact that for a randomly
sampled vector $z$, StyleGAN is about
$5$ times more likely to produce an image of a person without, than with glasses.

We highlight that, unlike regular datasets, the elements of
these datasets only have an interpretation \textit{given an already trained
  generator} $g$. In contrast to collected datasets embedded in $\RR^{1024
  \times 1024 \times 3}$ these collected points in $\RR^{512}$ do not have a
canonical interpretation as an image of a person, but rather depend on the
trained generator $g: \sgzo \rightarrow \sgimg$.

Usage of a pretrained generator allows us to separate two notions: (i) learning
the product task, and (ii) learning to generate realistic images of a given class. 
By using a large pretrained network such as StyleGAN to do the actual
image generation, the available computing power can be used solely for learning
the isomorphism in Figure \ref{fig:birdseye} (b).
As a result, all generated images \textit{look realistic from the start}. The
composition and decomposition networks quickly learn to constrain their image to
the part of the latent space StyleGAN generator is trained on, thus making the
output as realistic as an average sample of StyleGAN. The rest of the time is
spent learning \textit{which subset} of the latent space they should focus on.
This translates to choosing what type of person and glasses will be generated.
Before describing the details of this experiment, we show some a selection of results in Figure \ref{fig:showcase}.

\begin{figure}[H]
\begin{center}
  {\includegraphics[width=\linewidth]{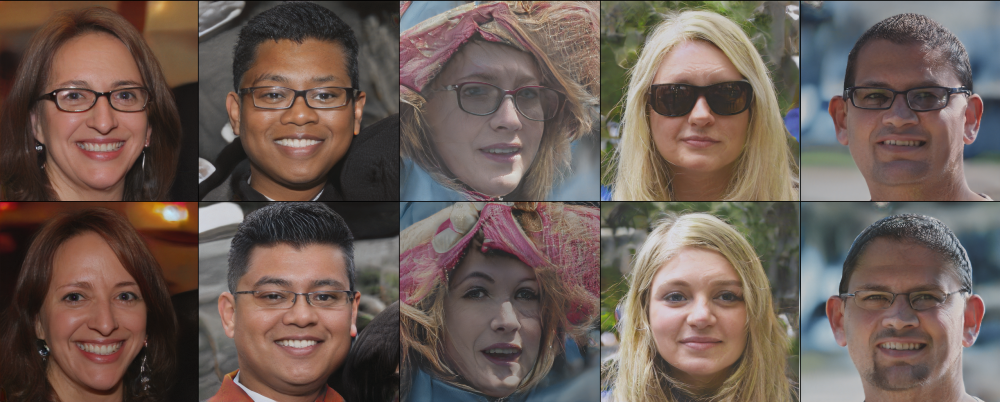}}
 \caption{Top row: input image of a person with glasses; Bottom row: output
   after the glasses have been removed. The entire learning is done in the \textit{latent
     space} of StyleGAN and these images each depict a latent vector $z$ after
   it has been applied to the generator $g$. Note how the network has learned to preserve
   facial features, without being specialized in any way for faces.}
\label{fig:showcase}
\end{center}
\end{figure}

We now restate the problem setting. Given the isomorphism in Figure
\ref{fig:birdseye} (b) we supply two datasets $D_E(AB)$ and $D_E(A)$ of latent
vectors of people with and without glasses, respectively. This dataset fixes the
embedding $\fr \xrightarrow{E} \set$ and thus the possible choice of architectures. The \textit{composition} network
architecture is an element of $\parae(\RR^{512} \times \RR^{512}, \RR^{512})$, while the
\textit{decomposition} network architecture is an element of $\parae(\RR^{512}, \RR^{512}
\times \RR^{512})$.

As these networks are generally orders of magnitude smaller than networks in
Section \ref{sec:celeba}, we had enough computing power at our disposal to try
out a variety of different network architectures and hyperparameter choices.
For both the generators and discriminators we have tested fully-connected,
convolutional and mixed architectures, ranging from just a couple of thousand
parameters to about half a million parameters. We have also tried a range of
choices of hyperparameters. That includes the learning rate, batch size, $\gamma$
(Definition \ref{def:total_loss}) and $n_{critic}$.

Using this dataset we have obtained interesting, but generally limited and
inconsistent results. Namely, we have not succeeded in finding a model that
solves the task in its entirety. All the trained models seem to have trouble learning
to optimize for both the adversarial and path-equivalence loss. Either the
models learn to preserve facial features or they learn to remove and insert
glasses, but never both. We especially note difficulties with the insertion of
glasses into the image.

In the Figures \ref{fig:stylegan_ab1}, \ref{fig:stylegan_ab2}, and \ref{fig:stylegan_a1} we use $c: \RR^{512} \times \RR^{512} \rightarrow
\RR^{512}$ and $d: \RR^{512} \rightarrow \RR^{512} \times \RR^{512}$ to refer to
a trained decomposition and composition network, respectively.
Even though in some samples we can see the model has learned to do an almost
perfect job, we have not been able to obtain uniformly positive results.

We hypothesize that the lack of these results could be for three reasons. The
first is that the product task regularization condition is not always sufficient
to successfully train the networks. Given some coherent choice of datasets, perhaps it is necessary to impose extra inductive biases to steer the training into the right
direction.
The second reason could be that any structure between points in the
\textit{image space} is not a preservation of this structure in the latent
space. Even though the generator produces realistic-looking images of people
with and without glasses, perhaps the coherence of this mapping is too strong of
an assumption. We hypothesize that initial training of the generator \textit{with} this regularization condition could improve the results.
The third possible reason is that we simply have not stumbled upon the right choice of
hyperparameters. This would be a sign that the proposed network is
not as robust as expected.

\begin{figure}[H]
\begin{center}
  {\includegraphics[width=\linewidth]{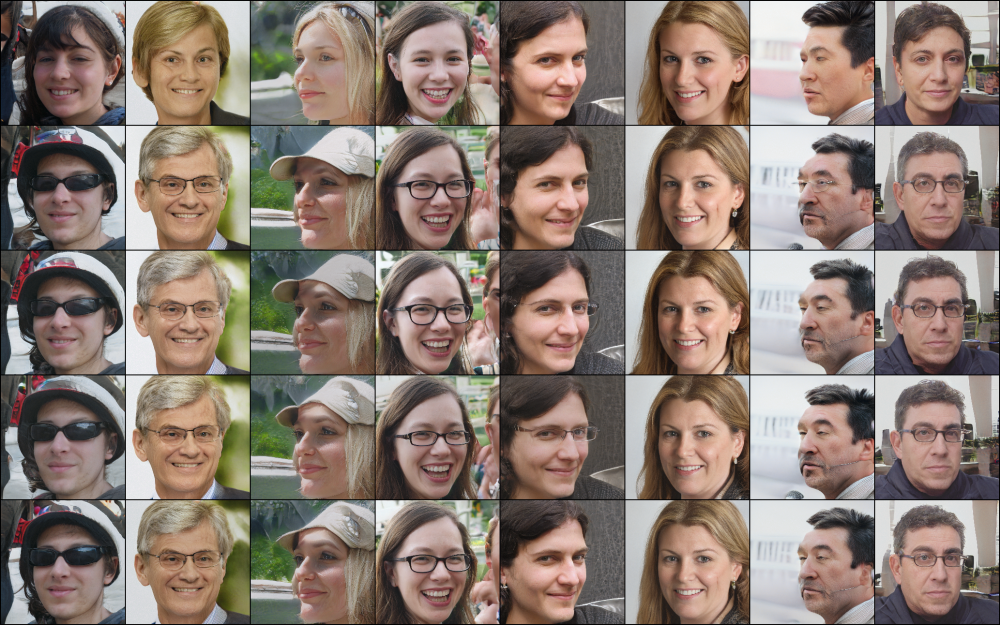}}
 \caption{Top row: input image $a_i$ of a person \textit{without} glasses. Every other
   row: the image $c(a_i, z_j)$, where $z_j$ is a per-row randomly sampled
   latent vector of glasses.
   Observe the failure of the model to learn to insert glasses on many of the
   faces. Even when it does learn to insert the glasses, it generates only
   minimal variation in glass shape, size and color, contrary to the expected.
   Furthermore, even though $z_j$ is fixed for each of the last three rows, the
   generated images contain substantially \textit{different} glasses for each person in a row.}
\label{fig:stylegan_a1}
\end{center}
\end{figure}

\begin{figure}[H]
\centering
\begin{subfigure}{1.0\linewidth}
  {\includegraphics[width=\linewidth]{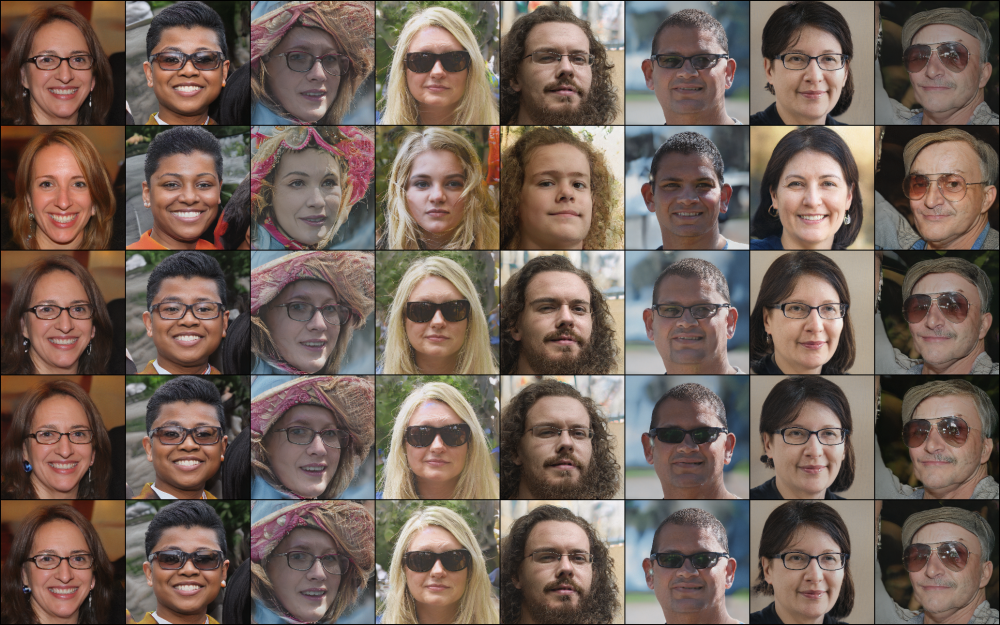}}
\caption{}
\end{subfigure}
\begin{subfigure}{1.0\linewidth}
  {\includegraphics[width=\linewidth]{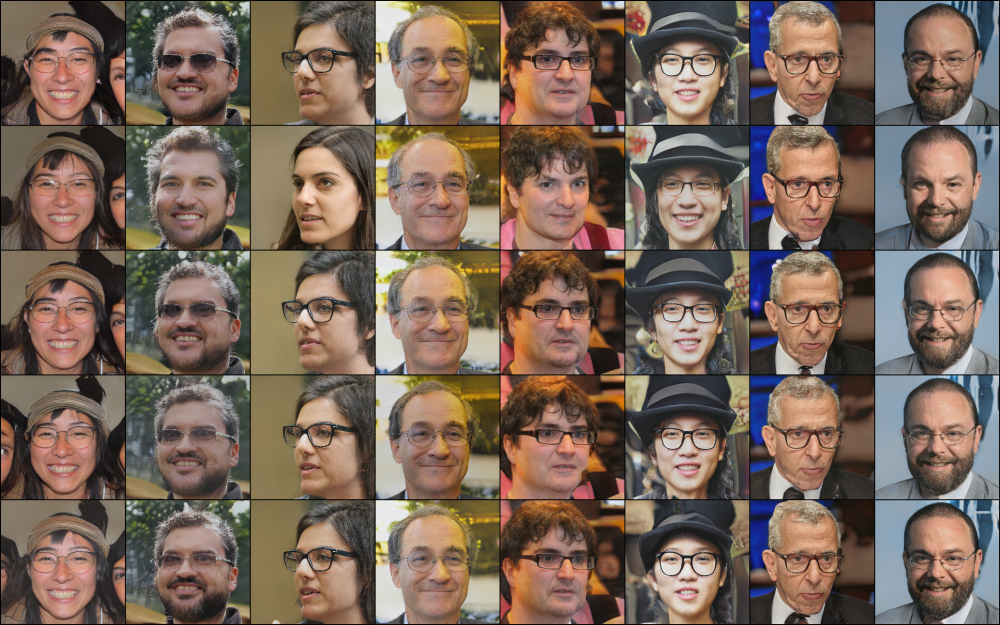}}
\caption{}
\end{subfigure}
 \caption{All subfigures depict the same mode on different input faces. Top row: input image $(ab)_i$ of a person with glasses; 2nd row:
   image $\pi_A(d((ab)_i))$ of that person after the glasses have been
   removed. Every other row: $c(\pi_A(d((ab)_i)), z_j)$ where $z_j$ is a per-row
   randomly sampled latent vector of glasses. Observe the failure of the model
 to coherently vary the generated glasses on a person as we change the latent
 vector $z_j$. 
Furthermore, even though $z_j$ is fixed for each of the last three rows, the
generated images contain substantially \textit{different} glasses for each
person in a row.}
\label{fig:stylegan_ab1}
\end{figure}

\begin{figure}[H]
\centering
\begin{subfigure}{1.0\linewidth}
  {\includegraphics[width=\linewidth]{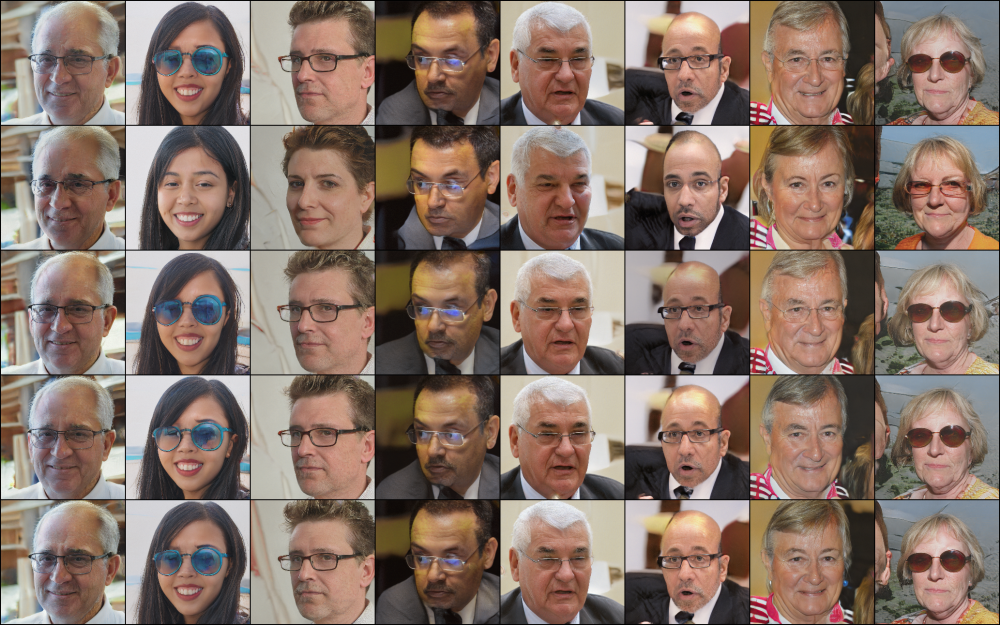}}
\caption{}
\end{subfigure}
\begin{subfigure}{1.0\linewidth}
  {\includegraphics[width=\linewidth]{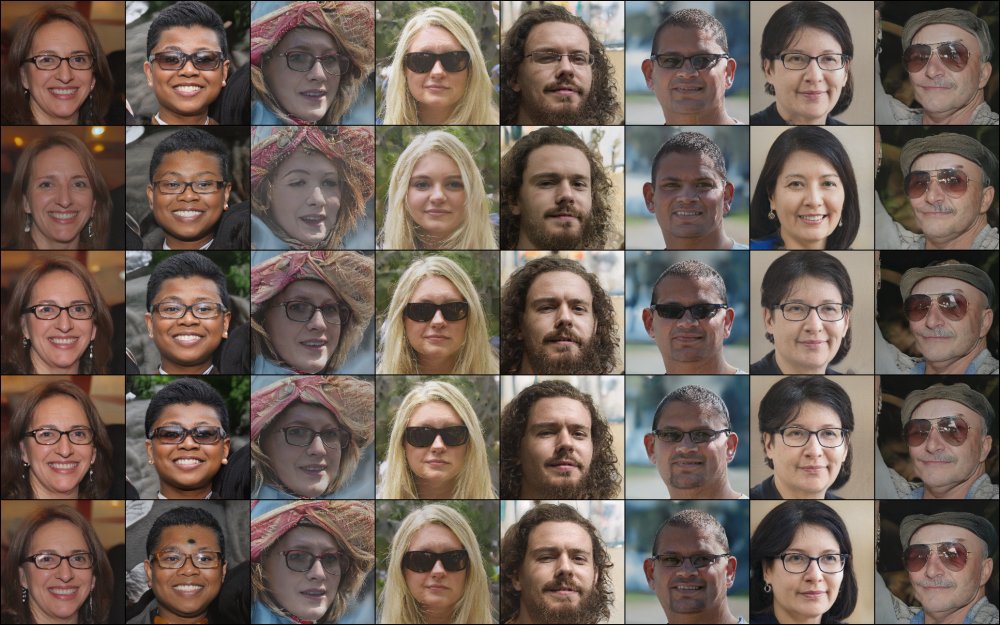}}
\caption{}
\end{subfigure}
 \caption{All subfigures depict the same mode on different input faces. Top row: input image $(ab)_i$ of a person with glasses; 2nd row:
   image $\pi_A(d((ab)_i))$ of that person after the glasses have been
   removed. Every other row: $c(\pi_A(d((ab)_i)), z_j)$ where $z_j$ is a per-row
   randomly sampled latent vector of glasses. Observe the failure of the model
 to coherently vary the generated glasses on a person as we change the latent
 vector $z_j$. 
Furthermore, even though $z_j$ is fixed for each of the last three rows, the generated images contain substantially \textit{different} glasses for each person in a row.
}
\label{fig:stylegan_ab2}
\end{figure}

\section{Experiment Summary}

With these three experiments, we have demonstrated the practical results obtained from
careful ascend into abstract categorical structures. In two out of the three
experiments we obtained promising results. Even though the last experiment
contains some samples of high quality in which careful removal and insertion of
glasses can be seen, we highlight that we have not been able to obtain uniformly positive
results.

With these experiments in mind, we believe the practical potential of
the product task has been shown. Simply by enforcing high-level rules about what ``removal''
and ``insertion'' mean in the product task, we have been able to train networks
to remove and insert objects in images without ever being told anything
about the nature of these images or objects in them.
Moreover, we highlight this is done with \textit{unpaired data}. As such, the high-level
regularization given by these path-equivalence losses proves to be a useful
training signal applicable to a wide variety of situations.

\subsubsection{Note on categorical schemas}

In this thesis our focus was on two simple schemas already given to us by
existing neural networks in Figure \ref{fig:birdseye}. In both cases there
exist clearly interpretable semantics and trainable network instances.
Even though the set of all possible schemas seems vast and potentially
interesting, we did not explore any other specific schemas.
It remains to be seen whether any other schemas capture new and interesting
semantics whose network instances can be efficiently trained.

% \chapter{Discussion}
% 
% We have seen how various deep learning abstractions can be embedded into a
% categorical framework. By interpreting neural networks in this light, we show
% this guides thought and provides many new fruitful directions for research.
% 
% \subsubsection{Learning limits using gradient descent}
% 
% Learning the \textit{product schema} is reminiscent of learning the unique
%   isomorphism between two limits of the index category $\fr := $
% \begin{tikzcd}[column sep=small, ampersand replacement=\&]
%   \LTO{A}\& \LTO{B} \\
% \end{tikzcd} into $\set$. One limit is the Cartesian product $A \times B$, the
% other one is the object $AB$ (as detailed in Section \ref{sec:product_schema}).
% Can such an isomorphism be learned between other limits? What would their
% interpretation be?  Can this insight pave way to a cleaner formalization of the product task and, generally, the notion of \textit{learning limits using gradient descent}?
% 
% A different avenue of attack for the this problem is presented in
% \cite{QueryLiftingProblems} where this issue is framed as a special case of a more general \textit{lifting problem}.

\chapter{Conclusion}

In this thesis we have begun to draw an outline of the rich categorical structure
underpinning deep learning.
Even though the category theory used is elementary, it is a versatile
tool which we use to do many things, including the provision of a tangible result.
As such category theory did not just help with formality but has guided our thought toward interesting questions.

We highlight that this endeavour is far from finished, for two reasons.
The first reason is that, even this narrow domain, there are plenty of
structures left to be put in their rightful place. Many of the
constructions -- especially those towards the end of the thesis -- lose their pure
categorical flavor.
The second reason is that we focus on a narrow subfield of deep learning, namely generative
adversarial models in the domain of computer vision. Other areas such
as recurrent neural networks, variational autoencoders, optimization, and
meta-learning are beyond the scope of this thesis and provide plausible
directions in which this study could be continued.

% The future of applied category theory seems bright. 
We have not fully reaped the payoff of the categorical approach. It is our hope
that in the coming years this thesis can serve as a part of the foundation of
further work which uses category theory to reason about the way we can teach machines to reason.

\bibliography{diplomski}
\bibliographystyle{fer}

\break
\hrtitle{Kompozicijsko duboko učenje}
\begin{sazetak}
Neuronske mreže postaju sve popularniji alat za rješavanje mnogih problema iz
stvarnoga svijeta.
Neuronske mreže općenita su metoda za diferencijabilnu optimizaciju koja
uključuje mnoge druge algoritme strojnog učenja kao specijalne slučajeve.
U ovom diplomskom radu izloženi su početci formalnog kompozicijskog okvira za razumijevanje različitih komponenata modernih arhitektura neuronskih mreža.
Jezik teorije kategorija korišten je za proširenje postojećeg rada o
kompozicijskom nadziranom učenju na područja nenadziranog učenja i generativnih
modela.
Prevođenjem arhitektura neuronskih mreža, skupova podataka, mape parameter-funkcija
i nekolicine drugih koncepata iz dubokog učenja u kategorijski jezik,
pokazano je da se optimizacija može raditi u prostoru funktora između dvije
fiksne kategorije umjesto u prostoru funkcija između dva skupa.
Dajemo pregled znakovite poveznice između formulacije dubokog učenja u ovom
diplomskom radu i formulacije kategorijskih baza podataka.
Nadalje, koristimo navedenu kategorijsku formulaciju kako bismo osmislili novu
arhitekturu neuronskih mreža kojoj je cilj naučiti umetanje i brisanje
objekata iz slike sa neuparenim podacima. Ispitivanjem te arhitekture na tri skupa
podataka dobivamo obećavajuće rezultate.

\kljucnerijeci{Neuronske mreže, duboko učenje, teorija kategorija, generativne
  suparničke mreže}
\end{sazetak}

% \engtitle{Compositional Deep Learning}
\begin{abstract}
Neural networks have become an increasingly popular tool for solving many
real-world problems.
They are a general framework for differentiable optimization which includes many
other machine learning approaches as special cases.
In this thesis we lay out the beginnings of a formal compositional framework for
reasoning about a number of components of modern neural network architectures.
The language of category theory is used to expand existing work on compositional
supervised learning into territories of unsupervised learning and generative
models. By translating neural network architectures, datasets,
parameter-function map, and a number of other concepts to the categorical setting,
we show optimization can be done in the functor space between two fixed
categories, rather than functions between two sets.
We outline a striking correspondence between the deep learning formulation in this thesis
and that of categorical database systems.
Furthermore, we use the category-theoretic framework to conceive a novel
neural network architecture whose goal is to learn the task of object insertion and
object deletion in images with unpaired data.
We test the architecture on three different datasets and obtain promising results.

\keywords{Neural networks, deep learning, category theory, generative
  adversarial networks}
\end{abstract}

\end{document}